\documentclass[twoside]{article}
\usepackage[accepted]{aistats2020}
\usepackage{times}
\usepackage{graphicx} % more modern
\usepackage{subfigure} 
\usepackage{wrapfig}
\usepackage{caption}

\usepackage{natbib}

\usepackage{algorithm}
\usepackage{algorithmic}

\usepackage{bm}
\usepackage{bbm}

\usepackage{amsmath, amsthm, amssymb}
\usepackage{mathtools}
\usepackage{multirow}
\usepackage{hhline}
\usepackage{booktabs}
\usepackage{multicol}
\usepackage{xcolor}

\usepackage[utf8]{inputenc} % allow utf-8 input
\usepackage[T1]{fontenc}    % use 8-bit T1 fonts
\usepackage{hyperref}       % hyperlinks
\usepackage{url}            % simple URL typesetting
\usepackage{amsfonts}       % blackboard math symbols
\usepackage{nicefrac}       % compact symbols for 1/2, etc.
\usepackage{microtype}      % microtypography

\newtheorem{theorem}{Theorem}
\newtheorem{lemma}{Lemma}

\newtheorem{definition}{Definition}
\newtheorem{remark}{Remark}
\newtheorem{assumption}{Assumption}[section]

\def\toptitlebar{\hrule height1pt \vskip 0.25in} 
\def\bottomtitlebar{\vskip 0.22in \hrule height1pt \vskip 0.3in} 
\let\oldemptyset\emptyset
\let\emptyset\varnothing

\DeclarePairedDelimiter\ceil{\lceil}{\rceil}
\DeclarePairedDelimiter\floor{\lfloor}{\rfloor}

\newcommand{\xdomain}{\ensuremath{\mathcal{X}}}
\newcommand{\R}{\ensuremath{\mathbb{R}}}
\renewcommand{\H}{\ensuremath{\mathcal{H}}}

\renewcommand{\algorithmicinput}{\textbf{Input:}}
\renewcommand{\algorithmicoutput}{\textbf{Output:}}

\begin{document}



\twocolumn[

\aistatstitle{LdSM: Logarithm-depth Streaming Multi-label Decision Trees}

\aistatsauthor{ Maryam Majzoubi \And Anna Choromanska}

\aistatsaddress{ New York University \And  New York University} 
]
\begin{abstract}
  We consider multi-label classification where the goal is to annotate each data point with the most relevant \textit{subset} of labels from an extremely large label set. Efficient annotation can be achieved with balanced tree predictors, i.e. trees with logarithmic-depth in the label complexity, whose leaves correspond to labels. Designing prediction mechanism with such trees for real data applications is non-trivial as it needs to accommodate sending examples to multiple leaves while at the same time sustain high prediction accuracy. In this paper we develop the LdSM algorithm for the construction and training of multi-label decision trees, where in every node of the tree we optimize a novel objective function that favors balanced splits, maintains high class purity of children nodes, and allows sending examples to multiple directions but with a penalty that prevents tree over-growth. Each node of the tree is trained once the previous node is completed leading to a streaming approach for training. We analyze the proposed objective theoretically and show that minimizing it leads to pure and balanced data splits. Furthermore, we show a boosting theorem that captures its connection to the multi-label classification error. Experimental results  on benchmark data sets demonstrate that our approach achieves high prediction accuracy and low prediction time and position LdSM as a competitive tool among existing state-of-the-art approaches. 
\end{abstract}

\section{INTRODUCTION}
\label{sec:Intro}
Plethora of modern machine learning approaches are concerned with performing multi-label predictions, as is the case in recommendation or ranking systems. In multi-label setting we receive examples $x \in \xdomain \subseteq \R^d$, with labels $y \subseteq \mathcal{Y} \equiv \{1,2,\ldots, K\}$, where each data point $x$ is assigned a \textit{subset} of labels $y$ from an extremely large label set $\mathcal{Y}$. This provides a generalization of the multi-class problem [\cite{BengioWG10,DengSBL11,TianshiGao:2011:DLR:2355573.2356555,NIPS2015_5937}], where each data point instead corresponds to a single mutually exclusive label.\footnote{It is non-trivial to extend multi-class trees to the multi-label setting [\cite{Prabhu18b}] as their training and prediction mechanism is not suitable for the setting when an example is equipped with more than one label.} Employing the label hierarchy, commonly represented as a tree with leaves corresponding to labels, potentially allows for faster prediction when the hierarchy is balanced and thus the tree depth is of size $\mathcal{O}(\log_M K)$ for $M$-ary tree, and enables overcoming the intractability problem of common baselines, such as one-against-all (OAA) [\cite{Rifkin2004}] that requires evaluating $K$ classifiers per example. Tree-based predictors are therefore commonly used, but since the label hierarchy is unavailable most of the times, it has to be learned from the data.

The performance of the multi-label tree-based system heavily hinges on the structure of the tree [\cite{NIPS2008_3583,Jain16}]. Some approaches [\cite{pmlr-v48-jasinska16,DBLP:conf/nips/WydmuchJKBD18}] assume arbitrary label hierarchy that is not learned. For example, PLT [\cite{pmlr-v48-jasinska16}] considers a sparse probability estimates for F-measure maximization conditioned on the label tree. Majority of techniques however carefully design a splitting criterion that is recursively applied in every node of the tree to partition the data. These criteria differ between commonly-used tree-based multi-label classification approaches. Multi-label Random Forest (MLRF) [\cite{Agrawal13}] uses \textit{information theoretic losses}, specifically the class entropy or the Gini index, to obtain label hierarchy. Sparse gradient boosted decision trees (GBDT-S) [\cite{pmlr-v70-si17a}] build a regression tree that fits the
residuals from the previous trees and uses the multi-label hinge or squared loss. 

FastXML [\cite{Prabhu14}], PFastreXML [\cite{Jain16}], and SwiftXML [\cite{Prabhu18}] (the last one focuses on the prediction task with partially revealed labels) constitute a family of methods that rely on \textit{ranking losses}. FastXML learns a hierarchy over the feature space, rather than the label space, relying on the intuition that in each region of the feature space only a small subset of labels is active. The node objective function there promotes generalizability via standard regression loss and rank-prioritization via normalized Discounted Cumulative Gain (nDCG) ensuring that relevant positive labels for each point are predicted with high ranks.  PFastreXML improves upon FastXML by replacing the nDCG loss with its propensity scored variant (the same is used in SwiftXML) which is unbiased to the missing classes and assigns higher rewards for accurate tail label predictions. None of the above techniques use balancing term in their objective. 

There also exist methods that construct tree classifiers by optimizing \textit{clustering loss} in nodes. Hierarchical $k$-means underlies CRAFTML [\cite{craftml}] and older approaches to multi-label classification such as LPSR [\cite{pmlr-v28-weston13}], and HOMER [\cite{HOMER}].

The approach we propose in this paper belongs to the family of purely tree-based methods. It  partitions tree nodes based on joint optimization in the feature and label space. The node split is based on a new objective function that  explores the correlation between both spaces by conditioning the learning of feature space partitioning with data label information. The objective applies to trees of arbitrary width. It explicitly enforces class purity of children nodes (i.e. points within a partition are likely to have similar labels whereas points across partitions are likely to have different labels). Moreover, it relies on having multiple (i.e. two for binary tree) regressors at each node and thus allows sending examples to multiple children nodes. Multi-way assignment of examples is however penalized to better control tree accuracy. Finally, the objective encourages balanced partitions to ensure efficient prediction. The objective function comes with theoretical guarantees. We show that optimizing the objective improves the purity and balancedness of the data splits in isolation, i.e. when respectively the balancedness and purity is fixed. We next analyze the connection of the proposed objective with the multi-label classification error. We prove that when the objective is perfectly optimized in every tree node it leads to zero-error. We generalize this observation to a setting when the objective is gradually optimized in every tree node, but may never reach the actual optimum. We first show that minimizing the objective is causing the monotonic decrease of the error with every split. Next we prove a much stronger statement given in the form of boosting theorem that relies on weakly optimizing the objective function at each node of the tree and show that our tree algorithm boosts the weak learners at the nodes to achieve any desirable multi-label classification error in a finite number of splits. The resulting tree construction-and-training algorithm, that we call LdSM, results in \textbf{L}ogarithmic-\textbf{d}epth trees that are trained in a \textbf{s}treaming fashion, i.e. node-by-node\footnote{When training each node we stream multiple times through the data before moving to the next node. After we move, we never go back to the previously trained ones. Thus we assume the data set is finite (but can be very large). This differs from the online setting. For distinction between streaming and online settings see~\cite{SD}.}, and achieve competitive performance to other state-of-the-art tree-based approaches, being accurate and efficient at prediction, on large \textbf{m}ulti-label classification problems. In summary, our proposed objective function, the resulting tree construction algorithm, and theoretical analysis are all new and constitute the contributions of our paper.

We next discuss other approaches for multi-label classification. They constitute a different family of methods than purely tree-based techniques that our method belongs to and thus are not directly relevant to our work. Those techniques include extensions of OAA [\cite{Babbar:2017:DDS:3018661.3018741,DBLP:conf/icml/YenHRZD16,DBLP:conf/kdd/YenHDRDX17,SDM2019,pmlr-v54-niculescu-mizil17a, Babbar2019, Prabhu18b, 2019arXiv190408249K}], deep learning methods [\cite{attentionXML,Liu:2017:DLE:3077136.3080834,Zhang:2018:DEM:3206025.3206030,DBLP:conf/icml/JerniteCS17}] and approaches for learning embeddings [\cite{conf/icml/BalasubramanianL12,pmlr-v28-bi13,NIPS2013_5083,NIPS2009_3824,Tai:2012:MCP:2344802.2344812,pmlr-v15-zhang11c,NIPS2012_4561,pmlr-v32-yu14,pmlr-v20-ferng11,Ji:2008:ESS:1401890.1401939,37180,pmlr-v32-linc14}]. 

The paper is organized as follows: Section~\ref{sec:OF} presents the objective function, Section~\ref{sec:TR} provides theoretical results, Section~\ref{sec:A} shows the algorithm for tree construction and training and explains how to perform testing using the tree, Section~\ref{sec:E} reports empirical results on benchmark multi-label data sets, and finally Section~\ref{sec:C} concludes the paper. Supplementary material contains proofs of theorems from Section~\ref{sec:TR}, additional pseudo-codes of algorithms from section~\ref{sec:A}, and additional experimental results.

\section{OBJECTIVE FUNCTION}
\label{sec:OF}

We next explain the design of the objective function for the tree of arbitrary width $M$, i.e. tree where each node has M children, and show a special case of a binary tree. Below we consider an arbitrary non-leaf node of the tree and thus omit node index in the notation.

In our setting, each node of the tree contains $M$ binary classifiers $h_j$, where $j = 1,2,\dots,M$. $h_j \in \H$, where $\H$ is the hypothesis class with linear regressors. Consider an arbitrary non-leaf node and let $\pi_i$ denote the normalized fraction of examples containing label $i$ in their label set reaching that node, where the multiplicative normalizing factor is an inverse of the average number of labels per example containing label $i$ in their label set (note that $\sum_{i=1}^K\pi_i = 1$). The node regressors are trained in such a way that $h_j(x) \geq 0.5$ means that the example $x$ is sent to the $j^{\text{th}}$ subtree of a node (thus sending example to more than one child is possible). To prevent examples from stucking inside the node, in case when $h_j(x) < 0.5\: \forall_{j=1,2,\dots,M}$ the example is sent to the child node corresponding to the highest margin, i.e. $\left(\arg\max_{j = 1,2,\dots, M}h_j(x)\right)^{\text{th}}$ child node.  Let $P_j = P(h_j(x)>0.5)$ be the probability that the example $x$ reaches child $j \in \{1,\;2,...,\;M \}$ and let $P_j^i = P(h_j(x)>0.5|i)$ denote the conditional probability of these event when the example belongs to class $i$. Note that i) $\sum_{j=1}^MP_j \geq 1$, ii) for any $i = 1,2,\dots,K$, $\sum_{j=1}^MP_j^i \geq 1$, and iii) $P_j=\sum_{i=1}^k\pi_iP_j^i$. 
The node splitting criterion is defined as follows

\vspace{-0.2in}
\begin{eqnarray}
\!\!\!\!J \!\!\!&\!\!\!\coloneqq\!\!\!&\!\!\! \underbrace{\sum_{j = 1}^M\sum_{l=j+1}^M\!\!\left|P_j \!-\! P_l\right|}_{\text{balancing term}}\label{eqn:Mobj}\\
\!\!\!&& \!\!\! \underbrace{\underbrace{-\lambda_1 \sum_{i=1}^K \sum_{j = 1}^M\sum_{l=j+1}^M\!\! \pi_i\left|P_j^i \!-\! P_l^i\right|}_{\text{class integrity term}}
+ \underbrace{\lambda_2 \left|\left(\sum_{j=1}^MP_j\right)\!-\!1\right|}_{\text{multi-way penalty}}}_{\text{purity term}}\nonumber.
\end{eqnarray}
\vspace{-0.1in}

where $\lambda_1$, and $\lambda_2$ are non-negative hyper-parameters. The \textit{balancing term} guards an even split of examples between children nodes and is minimized for the \textit{perfectly balanced split} when $P_1=P_2=...=P_M$. The \textit{class integrity term} ensures that examples belonging to the same class are not split between children nodes. This term is maximized when $\ceil{\frac{M}{2}}$ or $\floor{\frac{M}{2}}$ probabilities from among $P_1^i,P_2^i,\dots,P_M^i$ are equal to $1$ and the remaining ones are equal to $0$ for any $i = 1,2,\dots,K$. Thus at maximum, given any class $i$, the examples containing this class in their label set are not split between children, but they are instead simultaneously all sent to $\ceil{\frac{M}{2}}$ or $\floor{\frac{M}{2}}$ children. The third term in the objective aims at compensating this multi-way assignment of examples. The \textit{multi-way penalty} prevents sending examples to multiple directions too often. It is maximized when $\forall_{j=1,2,\dots,M}\: P_j = 1$ and minimized when $\sum_{j=1}^MP_j = 1$. Thus the \textit{purity term}, defined as the sum of the class integrity term and the multi-way penalty, is minimized for the \textit{perfectly pure split} when no example is sent to more than one children (in other words, this is when for any $i = 1,2,\dots,K$, $P_j^i = 1$ for one particular setting of $j$ and $P_j^i = 0$ for all other $j$s). 

In the binary case the objective then simplifies to the following form:
\begin{equation}
J \coloneqq \underbrace{\left| P_R - P_L\right|}_{\text{balancing term}} - \underbrace{\underbrace{\lambda_1 \sum_{i=1}^K \pi_i \left|P_R^i - P_L^i\right|}_{\text{class integrity term}} + \underbrace{\lambda_2\left|P_R + P_L-1\right|}_{\text{multi-way penalty}}}_{\text{purity term}},
\label{eq:Bobj}
\end{equation}
where $P_R$ and $P_R^i$ ($P_L$ and $P_L^i$) denote the probabilities that the example reaches right (left) child, marginally and conditional on class $i$ respectively.

We aim to \emph{minimize the objective $J$} to obtain
high quality partitions. We next show theoretical properties of the objective introduced in Equation~\ref{eqn:Mobj}.

\section{THEORETICAL RESULTS}
\label{sec:TR}

In this section we analyze the properties of the objective and its influence on the purity and balancedness of node splits. Next we show its connection to the multi-label error.

\subsection{General Properties of the Objective and its Relation to Node Partitions}
The two lemmas below provide the basic mathematical understanding of the objective $J$.
\begin{lemma}(Binary tree)
For any hypotheses $h_{R/L}\in\mathcal{H}$, the objective $J$ defined in Equation~\ref{eq:Bobj} satisfies $J\in[-\lambda_1,\;\lambda_2]$ and it is minimized if and only if the split is perfectly balanced and perfectly pure. 
\label{lem:binary}
\end{lemma}
Lemma~\ref{lem:binary} generalizes to the tree of arbitrary width as follows:
\begin{lemma} ($M$-ary tree)
For any hypotheses $h_j \in \mathcal{H}$, where $j = 1,2,\dots,M$, and sufficiently large $\lambda_2$, i.e. $(M-3 < \frac{\lambda_2}{\lambda_1})$, the objective $J$ defined in Equation~\ref{eqn:Mobj} satisfies $J\in[-\lambda_1(M-1),\;\lambda_2(M-1)]$ and it is minimized if and only if the split is perfectly balanced and perfectly pure. 
\label{lem:mary}
\end{lemma}

Let $J^*$ denote the lowest possible value of the objective $J$, i.e. $J^* = -\lambda_1(M-1)$.

Next we study how the objective promotes building nodes that are as balanced and pure as possible given the data. We first introduce useful definitions.

\begin{definition}(Balancedness)
The node split is $\beta$-balanced if the following holds
$
\max_{j=\{1,2,...,M \} } \left|P_j - \frac{\sum_{i=1}^M P_i}{M}\right| = \beta,
$
where $\beta\in\left[0, 1-\frac{1}{M}\right]$ is a balancedness factor.
\end{definition}
Note that a split is perfectly balanced if and only if $\beta=0$.

\begin{definition}(Purity)
The node split is $\alpha$-pure if the following holds
\begin{equation}
\frac{1}{M}\sum_{j=1}^M \sum_{i=1}^K \pi_i\min\left(P_j^i,\sum_{l=1}^M P_l^i-P_j^i\right) = \alpha,
\end{equation}
where, $\alpha\in[0, 1]$ is a purity factor.
\end{definition}
Note that a split is perfectly pure if and only if $\alpha=0$.\\
Next lemmas show that in isolation, when either the purity or balancedness of the split is fixed, decreasing the value of the objective leads to recovering more balanced or pure split, respectively.

\begin{lemma} If a node split has a fixed purity term $\alpha$, with corresponding $J_{\text{purity}}^{\alpha}$ then
$\beta \leq J - J_{\text{purity}}^{\alpha}$.
\label{lem:balancefix}
\vspace{-0.05in}
\end{lemma}

\begin{lemma}
If a node split has a fixed balanced term $\beta$, with corresponding $J_{\text{balance}}^{\beta}$ and assuming that the following condition holds: $\lambda_1(M-1) + J_{\text{balance}}^{\beta} \geq \lambda_2 \geq \lambda_1\frac{M-1}{2}$,  then 
\begin{equation}
\alpha \leq (J - J_{\text{balance}}^{\beta} +\lambda_2)\frac{2}{M(2\lambda_2-\lambda_1(M-1))}.
\end{equation}
\label{lem:purefix}
\vspace{-0.2in}
\end{lemma}

\subsection{Relation of the Objective to the Multi-label Error}

We will next explore the connection of the multi-label classification error with the proposed objective. For simplicity, assume each example has $r$ labels. Denote $t(x)$ to be the true label set of $x$ and $y_r(x)$ to be the assigned label set of size $r$ by the tree. Denote $e_r(\mathcal{T})$ to be the $r$-level error with respect to the Precision$@r$ measure, i.e. $e_r(\mathcal{T}) = 1 - \text{Precision}@r = \frac{1}{r}\sum_{i=1}^KP(i \in y_r(x), i \notin t(x))$. 

\subsubsection{Ideal Case}

Here we consider the ideal case when the objective $J$ is perfectly minimized in every node of the tree and show that in this case the tree achieves zero multi-label classification error.

\begin{theorem}
When the objective function $J$ from Equation~\ref{eqn:Mobj} is perfectly minimized in every node of the tree, i.e. $J = J^{*}$, then the resulting multi-label tree achieves zero $\hat{r}$-level multi-label error, $e_{\hat{r}}(\mathcal{T})$, for any $\hat{r} = 1,2,\dots,r$. 
\label{thm:error1}
\end{theorem}

\subsubsection{Real Case: Boosting Theorem}

Next we prove a bound on the classification error for the LdSM
tree. In particular, we show that if the proposed objective is “weakly”
optimized in each node of the tree then our algorithm will boost this weak advantage
to build a tree achieving any desired level of accuracy. This weak advantage is captured in a form of the Weak Hypothesis Assumptions. We restrict ourselves to the case of binary tree. We omit the analysis for the $M$-ary to avoid over-complicating the notation. 

We introduce the following weak assumptions.
\begin{assumption}
$\gamma$-Weak Hypothesis Assumption: for any distribution $\mathcal{P}$ over the data, at each node of the tree $\mathcal{T}$ there exist a partition such that
$\sum_i \pi_i \left|P_R^i - P_L^i\right| \geq \gamma$, where $\gamma \in (0,1]$.
\label{as:WHA1}
\end{assumption}
\begin{remark}
The above definition essentially assumes that in every node of the tree we are able to recover a partition with the corresponding class integrity term (a second component of our objective) bounded away from zero. Since the value of this term ranges in $[0,1]$, such assumption is indeed very ``weak''.
% \textcolor{red}{and is not satisfied only if for all labels the following holds: the examples with the same label are split evenly between children nodes. Optimizing the objective function prevents this term to be zero as explained in section~\ref{sec:OF}}. 

% \textcolor{blue}{(Comment: we have added some extra explanations in our response, but I feel it is too repetitive. I have commented it out.)}

Also, specifically note that it is enough that for one class $i$ the following holds: $|P_R^i - P_L^i| \geq \gamma$ in order to satisfy the assumption. 
\end{remark}

Assumption~\ref{as:WHA1} leads to the lemma that captures the monotonic drop of the error with each split. A similar theorem for $M$-ary case is provided in the Supplement. 

\begin{lemma} Under the Weak Hypothesis Assumption~\ref{as:WHA1}, $e_r(\mathcal{T})$ is monotonically decreasing with every split of the tree.
\label{lem:error2}
\end{lemma}

We next introduce the second weak assumption.

\begin{assumption}
$c$-Weak Hypothesis Assumption: at step $t$ of the algorithm, there exist a leaf node $l^*$ such that its weight, $w_{l^*} \geq \frac{c}{(t + 1)}$, where $c \in (0, 1]$. $w_{l^*}$ is the probability that a randomly chosen point from distribution $\mathcal{P}$ reaches the leaf $l^*$.
\label{as:WHA2}
\end{assumption}

\begin{remark}
Note that at step $t$ of the algorithm we have $t+1$ leaves. Also note that $\sum_{\tilde{\mathcal{L}} \subset \mathcal{L}} w_{\tilde{\mathcal{L}}} = 1$, where $\tilde{\mathcal{L}}$ is a subset of the tree leaves ($\mathcal{L}$ is the set of all leaves and the sum is taken over all subsets of tree leaves) and $w_{\tilde{\mathcal{L}}}$ is the weight of this subset, or equivalently, the probability that a randomly chosen point from distribution $\mathcal{P}$ reaches all leaves in $\tilde{\mathcal{L}}$. 

Consider the case when we do not send examples to more than one direction in every node of the tree. In this case there exists a leaf $l^*$ with c equal to 1 and therefore $w_{l^*} \geq 1/(t+1)$. When c decreases, we allow to send more examples to multiple directions in the tree. In the Assumption~\ref{as:WHA2} we let the examples to be sent to both directions at some nodes, therefore we require that there exists a leaf with $w_{l} \geq c/(t+1)$, for $c < 1$. Thus Assumption~\ref{as:WHA2} is tightly correlated with the multi-way penalty term of the objective. Note also that naturally, c has to be bounded away from zero since every leaf receives at least one example.
\end{remark}

Then the following theorem holds.
\begin{theorem}
Under the Weak Hypothesis Assumptions~\ref{as:WHA1} and~\ref{as:WHA2} and an additional assumption that each node produces perfectly balanced split,  for any $\alpha \in [0,1]$ to obtain $e_r(\mathcal{T}) \leq \alpha$ it suffices to have a tree with $t$ internal nodes that satisfy
\begin{equation}
  (t+1)\geq (\frac{1}{\alpha})^{\frac{16 \ln K}{c r^2 \gamma^2 (1- b) \log_2(e)}}, 
\end{equation}
where $b = |P_R + P_L - 1|$.
\label{thm:errorbound}
\end{theorem}
Consider an algorithm that builds the tree in a top-down fashion so that at each step it chooses the node with the highest weight, optimizes $J$ at that node, and splits that node to its children. The theorem does not assume that we can optimize $J$ perfectly in the node but instead only requires weak assumptions to hold. The above theorem guarantees that we can amplify the weak gain at each node to decrease the error below any desirable threshold. In practice we expect that $J$ can be optimized far better than what is given by the weak hypothesis assumptions, which effectively reduces the number of required splits needed to achieve given multi-label classification error. A generalization of Theorem~\ref{thm:errorbound} is provided in the Supplement (note it relies on more complicated assumptions though).

\section{ALGORITHM}
\label{sec:A}
In this section we present the algorithm for simultaneous tree construction and training. We then discuss how to assign labels to the test example. The main algorithm for tree construction and training is captured in Algorithm~\ref{alg:B}. It presents the top-level procedure for building the tree. It includes three sub-algorithms which we will explain here but their pseudo-codes are deferred to the Supplement. The tree construction is performed in a top-down node-by-node fashion. Reaching the maximum number of nodes terminates further growth of the tree. As can be seen in Algorithm~\ref{alg:B}, we select a node to be expanded into children nodes based on the priority computed as the difference of the sum and maximum value of the bins of the label histogram in the node. The priority of the node is related to the weight of the node defined in section~\ref{sec:TR}. We want to split nodes that are reached by many examples but we also require them to come from different classes, where at least two classes have significant mass. High priority is attained by these nodes that were visited by many examples that correspond to many different labels.
When the node is selected for expansion, we train its regressors according to the procedure \textbf{TrainRegressors} (see Algorithm~\ref{alg:WU} in the Supplement for its pseudo-code). In \textbf{TrainRegressors} we stream multiple (\#epochs) times through the data reaching that node and optimize the objective function for each example according to the procedure \textbf{OptimizeObjective} (see Algorithm~\ref{alg:OO} in the Supplement). In \textbf{OptimizeObjective} we search over all possible ways of sending an example to $M$ directions (including multi-way cases) and we choose the set of directions for which J achieves the lowest value. \textbf{TrainRegressors} uses these computed optimal direction(s) to train its regressors using cross-entropy loss. Afterwards, it updates the probabilities $P_j$s and $P_j^i$s in the node. Instead of taking $1$-increments per example when  updating probabilities, we  use regressor margins (clamped to the interval $[0,1]$). 

After training the regressors, we create children for the node according to the procedure \textbf{CreateChildren} (see Algorithm~\ref{alg:CC} in the Supplement).
Based on the outputs of the regressors we assign data points to its children using rule explained in Section~\ref{sec:OF} and update children's label histograms accordingly. 

At testing, the prediction is formed according to Algorithm~\ref{alg:P}. Specifically, the example is sent down the tree, from the root to one or more leaves, guided by node regressors. For examples that descended to multiple leaves, we estimate the label histogram by averaging the normalized label histograms of these leaves. The normalized label histogram is computed by dividing the label histogram by the sum of its entries. Given $R$ (the input to the Algorithm~\ref{alg:P}), we assign to the test example top $R$ labels that correspond to the highest entries in the resulting histogram.

% \begin{minipage}[t]{0.48\textwidth}
\vspace{-0.05in}
\begin{algorithm}[H]
\caption{\textbf{BuildTree}}
\begin{algorithmic} 
\STATE \% $v.I$ denotes the list of indices of examples\\
\hspace{0.16in}reaching node $v$
\INPUT $\cdot$ maximum \# of nodes: $T_{\text{max}}$;\\
\hspace{0.31in}$\cdot$ tree width: $M$;\\
\hspace{0.31in}$\cdot$ \# of training epochs: $E$;\\ 
\hspace{0.31in}$\cdot$ training data $(x_1,y_1),\dots,(x_N,y_N)$\\
\hspace{0.34in} \%$y_i$: all labels of the $i^{\text{th}}$example
\\\hrulefill
\STATE {\bf procedure UpdateHist ($LHist$, $y$)}\\
\textbf{for} $i \in y$ \textbf{do}\;\;
$LHist[i] \mathrel{+}= 1$ \;\;\textbf{end for}
% \FOR{$i \in y$} 
% \STATE $LHist[i] \leftarrow LHist[i] + 1$ 
% \ENDFOR
\\\hrulefill
\STATE
\vspace{-0.1in}
\STATE $v_{root}.I \leftarrow \{1,2,...,N\}$; \:\:\:$v_{root}.Lhist \leftarrow \oldemptyset$
\FOR{$i \in v_{root}.I$}
\STATE \% add $y_i$ to histogram
\STATE {\bf UpdateHist ($v_{root}.Lhist$, $y_i$)}
\ENDFOR
\STATE $t \leftarrow 1$
\STATE $Q.push(v_{root}, 0)$ \% initialize priority queue $Q$\\
\WHILE{$Q \neq \oldemptyset$ \AND $t < Tmax$}
\STATE $v \leftarrow Q.pop()$ 
\STATE \textbf{TrainRegressors ($v$)}
\STATE $ch \leftarrow$ \textbf{CreateChildren ($v$)} 
\FOR{$m \in ch$} 
\STATE $priority \leftarrow$\\ 
\hspace{0.15in}$\sum_{k \in ch[m].Lhist}{ch[m].Lhist[k]}$
\STATE $\hspace{0.18in}- \max_{k \in ch[m].Lhist}{ch[m].Lhist[k]}$
\STATE $Q.push(ch[m], priority)$ 
\ENDFOR
\STATE $t \leftarrow t + M$
\ENDWHILE
\STATE {\bf return} $v_{root}$ 
\end{algorithmic}
\label{alg:B}
\end{algorithm}
\vspace{-0.15in}

\paragraph{The computational complexity analysis} The complexity of the \textbf{TrainRegressors} is $\mathcal{O}(M(D+K)+eM\hat{N}(\hat{d}+ 2^M\hat{k}))$, where $D$ is the feature size, $K$ is the label size, $e$ is the number of epochs, $\hat{N}$ is the number of examples reaching the node, $\hat{d}$ is the average number of features per point and $\hat{k}$ is the average number of labels per point. The first term, $M(D+K)$, only corresponds to the initialization of the regressors and conditional probabilities. Note that since the feature and label spaces are sparse, $\hat{d}$ and $\hat{k}$ are small numbers compared to $D$ and $K$. If we expect $e\hat{N}(\hat{d}+2^M \hat{k}) << (D+K)$, the complexity can be further reduced when using a self-balancing binary search tree to store the sparse set of weight vectors and probabilities. This would result in $\tilde{\mathcal{O}}(eM\hat{N}(\hat{d} + 2^M\hat{k}))$ computational complexity ($M$ is usually a small number. In our experiments we used $M=2,4$). Let $\hat{r}$ be the average number of leaves that each example descends to. Then the overall training complexity when building a balanced tree would become $\tilde{\mathcal{O}}(N\hat{r}eM(2^M\hat{k} + \hat{d}))$, where $N$ is the the size of the training data. Furthermore, the testing complexity would become $\mathcal{O}(\log(K)\hat{r}M(\hat{k}+\hat{d}))$ per test data point. Note that having an explicit balancing term in our objective function encourages building trees with logarithmic depth with respect to the total number of labels, $K$.
\section{EXPERIMENTS}
\label{sec:E}
We evaluated LdSM on multiple benchmark data sets (Bibtex, Mediamill, Delicious, AmazonCat-13k, Wiki10-31k, Delicious-200K, and Amazon-670k) obtained from public repository~\cite{Manik}. The data sizes are reported in Table~\ref{tab:accu} ($D$ is the data dimensionality). The experimental setup is described in the Supplement.

% \begin{minipage}{0.5\textwidth}
\vspace{-0.05in}
\begin{algorithm}[H]
\caption{\textbf{Predict ($x$, $R$)}}
\begin{algorithmic} 
% \STATE \% $R$ denotes the number of labels to\\ \hspace{0.16in}predict per example
\INPUT $\cdot$ root of the trained tree: $v_{root}$;\\ 
\hspace{0.31in}$\cdot$ \#labels to predict per example: $R$;\\ \hspace{0.31in} $\cdot$ tree width: $M$
\\\hrulefill
\STATE {\bf procedure GetLeaves($v$)}
\IF{$v.isLeaf$}
\STATE $leafList.push(v)$
\ELSE
\FOR{$m \in 1\dots M$}
\IF{$v.w_m^{\top}x>0.5$}
\STATE {\bf GetLeaves($v_m$)}
\STATE $sent \leftarrow true$
\ENDIF
\ENDFOR
\IF{{\bf not} $sent$}
\STATE $m \leftarrow \arg\max_{\hat{m}\in \{1,2,\dots,M\}}{v.w_{\hat{m}}^{\top}x}$
\STATE {\bf GetLeaves($v_m$)}
\ENDIF
\ENDIF
\\\hrulefill
\STATE $leafList \leftarrow \oldemptyset$ \% list of leaves reached by example $x$
\STATE {\bf GetLeaves($v_{root}$)}
\STATE $hist \leftarrow \oldemptyset$ 
\FOR{$v_l \in leafList$}
\STATE $sum \leftarrow \sum_{k \in v_l.Lhist}v_l.Lhist[k]$
\FOR{$k \in v_l.Lhist$}
\STATE $hist[k] \mathrel{+}= v_l.Lhist[k]/sum$
% \hspace{0.2in}$\leftarrow hist[k] + v_l.Lhist[k] / sum$
\ENDFOR
\ENDFOR
\STATE $labels \leftarrow$ select $R$ top entries from $hist$
\STATE {\bf return} $labels$
\end{algorithmic}
\label{alg:P}
\end{algorithm}
\setlength{\textfloatsep}{0cm}
\setlength{\floatsep}{0cm}
\vspace{-0.15in}
% \end{minipage}
In Table~\ref{tab:accu} we compare the Precisions $P@1$, $P@3$, and $P@5$ and nDCG scores $N@1$, $N@3$, $N@5$ (see~\cite{Manik} for the explanation of these evaluation metrics) obtained by LdSM and other purely tree-based competitor algorithms: LPSR, FastXML, PFastreXML, PLT, GBDT-S, and CRAFTML.The performance of the competitors were obtained from the corresponding papers introducing these techniques and multi-label repository~\cite{Manik}. The prediction with LdSM ensemble is done by averaging the resulting histograms for each tree and then selecting $R$ labels. At training, each tree in the ensemble differs in regressors initialization. The reported results show that LdSM either matches or, on selected problems (including large Amazon-670k data set), outperforms the existing tree-based approaches in terms of both the Precision and the nDCG score. 

In Table~\ref{tab:pOAA} we report the Precisions of LdSM compared with Parabel, the most recent OAA approach, which is also efficient compared to the other schemes. Parabel builds a hierarchy over labels and also learns powerful 1-vs-All classifiers in the leaf nodes. It can be considered as a hybrid technique which combines 1-vs-All with label-tree approaches. It has much better prediction time compared to the other 1-vs-All approaches while achieving similar accuracies as DiSMEC/PPD-Sparse. The comparison with the rest of the techniques are deferred to the Supplement. On bigger data sets, LdSM has some loss of statistical accuracy with respect to OAA methods. However, these techniques have fundamentally different underlying mechanism from ours, which usually result in their higher complexity and longer prediction time. 

We observed the largest data set (Amazon-670k) suffer from the tail label problem. For this data set we use re-ranking approach similar to~\cite{Jain16}. This is applied at testing, after our tree is built and trained. Re-ranking increases the test time by $\sim50\%$ for Amazon-670k. In Table~\ref{tab:pOAAPS} we compare the performance of our approach against Parabel on tail labels using the propensity score variant of Precision. On most of the data sets LdSM has better performance.

In Table~\ref{tab:predtime} we provide per-example prediction time (training time is deferred to the Supplement) for different data sets comparing LdSM with competitor methods, as well as with Parabel. Our result demonstrates that LdSM can perform efficient multi-label prediction, with respect to the tree-based methods as well as the other techniques including OAA approaches. (Refer to the Supplement for more results.)

Figure~\ref{fig:depth} shows that the depth of trees constructed with LdSM are $\mathcal{O}(\log_M(K))$, specifically they lie in the interval $[\log_MK,3\log_MK]$ for Mediamill, Bibtex and Delicious-200k data sets and $[\log_MK,2\log_MK]$ for Delicious, AmazonCat-13k, Wiki10-31k and Amazon-670k data sets. Next we discuss the results captured in Figure~\ref{fig:impex}. Note that additional figures related to this study can be found in the Supplement. In the top plot we report the behavior of Precision and nDCG score as the size of the LdSM ensemble grows. Clearly the most rapid improvement in Precision is achieved when increasing the ensemble size to $10$ trees (across different data sets this was found to be between $5$ and $10$, except Bibtex (case $M=2$), for which it was $20$). After that, the increase of $P@1$, $P@3$, $P@5$, $N@1$, $N@3$, and $N@5$ saturates and we obtain less than $2\%$ improvement when increasing the ensemble further to $50$. The same can be observed for nDCG score. The bottom plot captures how the Precision and nDCG score depend on the number of nodes in the tree and the depth of the deepest tree in the ensemble. As we increase the maximum allowed number of nodes ($T_{max}$) in the LdSM algorithm, it recovers $\mathcal{O}(log_M(T_{max}))$-depth trees. One can observe the general tendency that increasing the number of nodes $\phi$ times, results in increasing the tree depth by less than $2\log_M(\phi)$. We also observed that increasing the number of nodes/tree depth for most data sets leads to the improvement in Precisions $P@1$, $P@3$, and $P@5$ and nDCG scores $N@1$, $N@3$, and $N@5$ by less than $3\%$, suggesting that often shallower trees already achieve acceptable performance. The plots in Figure~\ref{fig:impex5} in the Supplement demonstrate that single LdSM tree outperforms single FastXML tree. The same property holds for ensembles. 

In Figure~\ref{fig:OC} we show how the objective function is optimized as we move from the root deeper into the tree. Intuitively root faces the most difficult optimization task as it sees the entire data set and consequently the objective function there is optimized more weakly, i.e. to a higher level, than in case of nodes lying deeper in the tree. As we move closer to the leaves, the convergence is faster due to the ``cleaner'' nature of the data received by the nodes there (less label variety).

% \begin{minipage}{0.47\linewidth}
\begin{minipage}{0.45\textwidth}
% \begin{figure}[h]
    \begin{center}
    % \vspace{0.5in}
    % \includegraphics[width=1\textwidth]{Figures/deli-root.png}
    \includegraphics[width=1\textwidth]{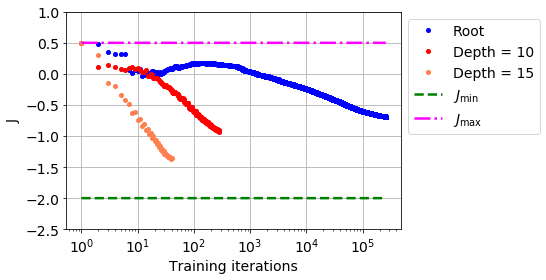}
    \end{center}
    \vspace{-0.15in}
    \captionof{figure}{The behavior of the LdSM objective function $J$ during training at different levels in the tree for an exemplary LdSM tree. Delicious data set. Tree depth is $20$ and $M$ was set to $M = 2$. $J_{min}$ and $J_{max}$ denote respectively the minimum and maximum value of $J$.}
    \label{fig:OC}
    % \vspace{0.1in}
% \end{figure}
\end{minipage}

\begin{figure}[htp!]
%   \vspace{-0.1in}
  \begin{center}
\includegraphics[width=0.48\textwidth]{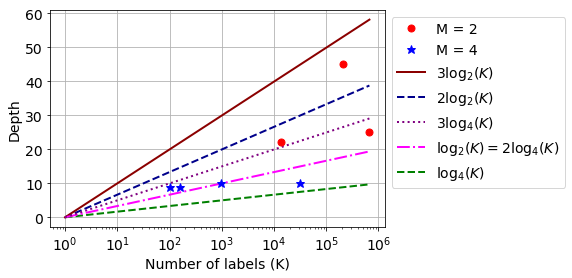}
\end{center}
\vspace{-0.1in}
\caption{The depth of the deepest tree in the optimal LdSM tree ensemble (reported in Table~\ref{tab:accu}) versus the number of labels in the data set ($K$).}
\label{fig:depth}
\vspace{-0.05in}  
\end{figure}

\begin{table}[htp!]
\vspace{-0.09in}
% \scriptsize
\centering
\setlength{\tabcolsep}{2pt}
\caption{Prediction time [ms] per example for tree-based methods as well as Parabel on different data sets  (LPSR and PLT are NA).}
\vspace{-0.2in}
\centering
\begin{tabular}{c||c|| c|c|c| c|| c|}
\cline{2-7} 
\multirow{2}{*}{} &\rotatebox{90}{Parabel}& \rotatebox{90}{GBDT-S}&\rotatebox{90}{CRAFTML}&\rotatebox{90}{FastXML} & \rotatebox{90}{PFastreXML} & \rotatebox{90}{LdSM} \\
\hline
\hline
\multicolumn{1}{|c||}{\multirow{1}{*}{Mediamill}} &NA&\textbf{0.05}&NA& 0.27 & 0.37 & \textbf{0.05} \\
\hline
\multicolumn{1}{|c||}{\multirow{1}{*}{Bibtex}}&NA &NA&NA& 0.64 & 0.73 & \textbf{0.013}\\
\hline
\multicolumn{1}{|c||}{\multirow{1}{*}{Delicious}}&NA &0.04&NA& NA & NA & \textbf{0.014}\\
\hline
% \multicolumn{1}{|c||}{\multirow{1}{*}{Eurlex-4k}} &NA&5.39& 3.65 & 5.43 & \textbf{0.03}\\
% \hline
\multicolumn{1}{|c||}{\multirow{1}{*}{AmazonCat-13k}}&NA &NA&5.12& 1.21 & 1.34 & \textbf{0.04}\\
\hline
\multicolumn{1}{|c||}{\multirow{1}{*}{Wiki10-31k}}&NA &0.20&NA& 1.38 & NA & \textbf{0.15}\\
\hline
\multicolumn{1}{|c||}{\multirow{1}{*}{Delicious-200k}}&NA &\textbf{0.14}&8.6& 1.28 & 7.40 & 1.21\\
\hline
% \multicolumn{1}{|c||}{\multirow{1}{*}{WikiLSHTC-325k}}&1.17 &NA&7.67& 1.02 & 1.47 & \textbf{0.81}\\
% \hline
\multicolumn{1}{|c||}{\multirow{1}{*}{Amazon-670k}}&1.13 &NA&5.02& 1.48 & 1.98 & \textbf{0.12}\\
\hline
\end{tabular}
\label{tab:predtime}
% \vspace{-0.1in}
\end{table}

\begin{figure}[h]
\vspace{-0.05in}
  \begin{center}
\includegraphics[width=0.45\textwidth]{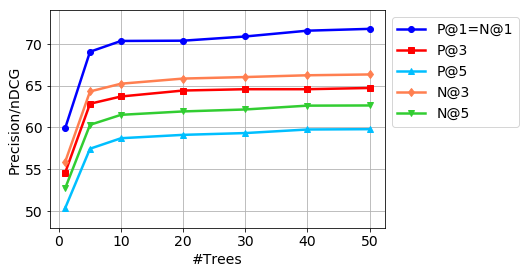}
\includegraphics[width=0.45\textwidth]{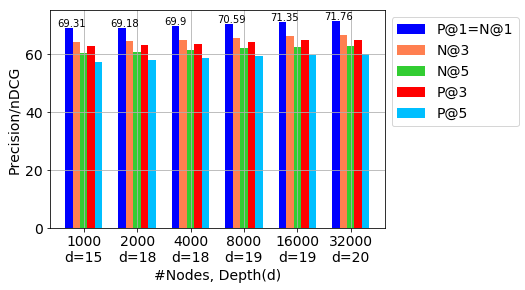}\\
\end{center}
\vspace{-0.05in}
\caption{The behavior of precision/nDCG score as a function of the number of trees in the ensemble (\textbf{first row}) and number of nodes $T_{max}$ (including leaves) and tree depth of the deepest tree in the ensemble (\textbf{second row}). $M$ is set to $M=2$. 
% \textbf{Last row}: The comparison of Precision for LdSM and FastXML working in the ensemble (\textbf{right bars}) as well as for single-tree (\textbf{left bars}) (LdSM-1: exemplary tree chosen from LdSM ensemble, LdSM-1$^{*}$, FastXML-1$^*$: optimal single trees). FastXML was the best performer from among LdSM's competitors on this data set.
Plots were obtained for Delicious data set (Table 1(c)). The figure should be read in color.
}
\label{fig:impex}
% \vspace{-0.2in}
\end{figure}
\section{CONCLUSIONS}
\label{sec:C}

This paper develops a new decision tree algorithm, that we call LdSM, for multi-label classification problem. The technical contributions of this work include: a novel objective function and its corresponding theoretical analysis and a resulting novel algorithm for tree construction and training that we evaluate empirically. We find experimentally that LdSM is competitive to the state-of-the art multi-label approaches, \;performs\; efficient\; prediction, \;and \;achieves 

\clearpage
\newpage
\begin{table}[htp!]
% \scriptsize
% \vspace{-0.05in}
\caption{Precisions: $P@1$, $P@3$, and $P@5$ ($\%$) and nDCG scores: $N@1$, $N@3$, and $N@5$ ($\%$) obtained by different tree-based methods on common multi-label data sets.}
\label{tab:accu}
\centering
\setlength{\tabcolsep}{1.5pt}
\begin{tabular}{c|| c| c| c||c| c| c|}
\hline
\multicolumn{7}{|c|}{(a) Mediamill $D = 120, K = 101$ }\\
\hline
\multicolumn{1}{|c||}{\multirow{1}{*}{Algorithm}} & P@1 & P@3 & P@5 & N@1 & N@3 & N@5\\
\hline
\multicolumn{1}{|c||}{\multirow{1}{*}{LPSR}} & 83.57& 65.78 & 49.97 & 83.57& 74.06 & 69.34\\
\hline
\multicolumn{1}{|c||}{\multirow{1}{*}{PLT}} & -& - & - & -& -& -\\
\hline
\multicolumn{1}{|c||}{\multirow{1}{*}{GBDT-S}} & 84.23& 67.85 & - & -& -& -\\
\hline
\multicolumn{1}{|c||}{\multirow{1}{*}{CRAFTML}} & 85.86& 69.01 & 54.65 & -& -& -\\
\hline
\multicolumn{1}{|c||}{\multirow{1}{*}{FastXML}} & 84.22& 67.33 & 53.04 &84.22 & 75.41& 72.37\\
\hline
\multicolumn{1}{|c||}{\multirow{1}{*}{PFastreXML}} &83.98 & 67.37 & 53.02 & 83.98& 75.31&72.21 \\
\hline
\hline
\multicolumn{1}{|c||}{\multirow{1}{*}{LdSM}}& \textbf{90.64} & \textbf{73.60} & \textbf{58.62} &\textbf{90.64} &\textbf{82.14}&\textbf{79.23} \\
\hline
\end{tabular}

\vspace{0.125in}
\setlength{\tabcolsep}{1.5pt}
% \hfill
\begin{tabular}{c|| c| c| c||c| c| c|}
\hline
\multicolumn{7}{|c|}{(b) Bibtex $D = 1.8k, K = 159$ }\\
\hline
\multicolumn{1}{|c||}{\multirow{1}{*}{Algorithm}} &  P@1 & P@3 & P@5 & N@1 & N@3 & N@5 \\
\hline
\multicolumn{1}{|c||}{\multirow{1}{*}{LPSR}} &  62.11& 36.65 & 26.53 &62.11 & 56.50& 58.23\\
\hline
\multicolumn{1}{|c||}{\multirow{1}{*}{PLT}} &   -& - & - & -& -& -\\
\hline
\multicolumn{1}{|c||}{\multirow{1}{*}{GBDT-S}} &  -& - & - & -& -& -\\
\hline
\multicolumn{1}{|c||}{\multirow{1}{*}{CRAFTML}} & \textbf{65.15}& \textbf{39.83} & 28.99& -& - &-\\
\hline
\multicolumn{1}{|c||}{\multirow{1}{*}{FastXML}} &  63.42& 39.23 & 28.86 & 63.42& 59.51&61.70 \\
\hline
\multicolumn{1}{|c||}{\multirow{1}{*}{PFastreXML}} &  63.46 & 39.22 & 29.14 & 63.46& 59.61& 62.12\\
\hline
\hline
\multicolumn{1}{|c||}{\multirow{1}{*}{LdSM}} &  64.69 & 39.70 & \textbf{29.25} &\textbf{64.69} &\textbf{60.37} &\textbf{62.73}  \\
\hline
\end{tabular}
\vspace{0.125in}

\setlength{\tabcolsep}{1.5pt}
\centering
\begin{tabular}{c|| c| c| c||c| c| c|}
\hline
\multicolumn{7}{|c|}{(c) Delicious $D = 500, K = 983$ }\\
\hline
\multicolumn{1}{|c||}{\multirow{1}{*}{Algorithm}} & P@1 & P@3 & P@5 & N@1 & N@3 & N@5\\
\hline
\multicolumn{1}{|c||}{\multirow{1}{*}{LPSR}} & 65.01&  58.96& 53.49 &65.01 & 60.45& 56.38\\
\hline
\multicolumn{1}{|c||}{\multirow{1}{*}{PLT}} & -& - & - & -& -& -\\
\hline
\multicolumn{1}{|c||}{\multirow{1}{*}{GBDT-S}} & 69.29& 63.62 & - & -& -& -\\
\hline
\multicolumn{1}{|c||}{\multirow{1}{*}{CRAFTML}} & 70.26& 63.98 & 59.00 & -& -& -\\
\hline
\multicolumn{1}{|c||}{\multirow{1}{*}{FastXML}} &69.61 & 64.12 & 59.27 &69.61 & 65.47&61.90 \\
\hline
\multicolumn{1}{|c||}{\multirow{1}{*}{PFastreXML}} & 67.13&  62.33& 58.62 &67.13 & 63.48& 60.74\\
\hline
\hline
\multicolumn{1}{|c||}{\multirow{1}{*}{LdSM}} & \textbf{71.91} & \textbf{65.34} & \textbf{60.24} &\textbf{71.91} & \textbf{66.90}&\textbf{63.09} \\
\hline
\end{tabular}
% \par
% \bigskip
% %\vspace{-0.05in}
% \vspace{0.05in}
% \hfill
% \setlength{\tabcolsep}{3pt}
% \begin{tabular}{||c|| c| c| c||c| c| c|}
% \hline
% \multicolumn{6}{||c|}{(d) Eurlex $D = 5k, K = 4K$}\\
% \hline
%  P@1 & P@3 & P@5 & N@1 & N@3 & N@5 \\
% \hline
%  76.37& 63.36 & 52.03 &76.37 &66.63 &60.61 \\
% \hline
%  -& - & - & -& -& -\\
% \hline
%  -& - & - & -& -& -\\
% \hline
% \textbf{78.81}& \textbf{65.21} & \textbf{53.71} & -& -& -\\
% \hline
% 71.36&  59.90& 50.39 &71.36 & 62.87& 58.06\\
% \hline
% 75.45& 62.70 & 52.51 & 75.45&65.97 & 60.78\\
% \hline
% \hline
% 75.91 & 61.86 & 50.90 & 75.91& 65.36& 59.47\\
% \hline
% \end{tabular}
% % \par
% % \bigskip
% % %\vspace{-0.05in}
% % % \vspace{0.05in}
% \hfill

\centering
\vspace{0.125in}
\setlength{\tabcolsep}{1.5pt}
\begin{tabular}{c|| c| c| c||c| c| c|}
\hline
\multicolumn{7}{|c|}{(d) AmazonCat-13k $D = 204k, K = 13k$}\\
\hline
\multicolumn{1}{|c||}{\multirow{1}{*}{Algorithm}} &P@1 & P@3 & P@5 & N@1 & N@3 & N@5 \\
\hline
\multicolumn{1}{|c||}{\multirow{1}{*}{LPSR}} &-& - & - &- &- &- \\
\hline
\multicolumn{1}{|c||}{\multirow{1}{*}{PLT}} &91.47& 75.84 & 61.02 & -& -& -\\
\hline
\multicolumn{1}{|c||}{\multirow{1}{*}{GBDT-S}} &-& - & - & -& -& -\\
\hline
\multicolumn{1}{|c||}{\multirow{1}{*}{CRAFTML}} &92.78& 78.48 & 63.58 & -& -& -\\
\hline
\multicolumn{1}{|c||}{\multirow{1}{*}{FastXML}} & 93.11&  \textbf{78.2}& 63.41 &93.11 & \textbf{87.07}&\textbf{85.16} \\
\hline
\multicolumn{1}{|c||}{\multirow{1}{*}{PFastreXML}} &91.75 &  77.97& \textbf{63.68} & 91.75& 86.48&84.96 \\
\hline
\hline
\multicolumn{1}{|c||}{\multirow{1}{*}{LdSM}} &\textbf{93.87}&75.41 &57.86 & \textbf{93.87}& 85.06&80.63 \\
\hline
\end{tabular}

\centering
\vspace{0.125in}
\setlength{\tabcolsep}{1.5pt}
\begin{tabular}{c|| c| c| c||c| c| c|}
\hline
\multicolumn{7}{|c|}{(e) Wiki10-31k $D = 102k, K = 31k$ }\\
\hline
\multicolumn{1}{|c||}{\multirow{1}{*}{Algorithm}}&P@1 & P@3 & P@5 & N@1 & N@3 & N@5 \\
\hline
\multicolumn{1}{|c||}{\multirow{1}{*}{LPSR}}&72.72&58.51  & 49.50 &72.72 & 61.71& 54.63\\
\hline
\multicolumn{1}{|c||}{\multirow{1}{*}{PLT}}&84.34& 72.34& 62.72& -& -& -\\
\hline
\multicolumn{1}{|c||}{\multirow{1}{*}{GBDT-S}}&84.34& 70.82 & - & -& -& -\\
\hline
\multicolumn{1}{|c||}{\multirow{1}{*}{CRAFTML}}&\textbf{85.19}& \textbf{73.17} & \textbf{63.27} & -& -& -\\
\hline
\multicolumn{1}{|c||}{\multirow{1}{*}{FastXML}}&83.03&  67.47& 57.76 &83.03 & \textbf{75.35}& 63.36\\
\hline
\multicolumn{1}{|c||}{\multirow{1}{*}{PFastreXML}}&83.57&  68.61& 59.10 & 83.57& 72.00&64.54 \\
\hline
\hline
\multicolumn{1}{|c||}{\multirow{1}{*}{LdSM}}&83.74 & 71.74 & 61.51 &\textbf{83.74} &74.60 & \textbf{66.77}\\
\hline
\end{tabular}
\end{table}
\hfill

\begin{table}[htp!]
\centering
\setlength{\tabcolsep}{1.5pt}
 \vspace{-0.2in}
\begin{tabular}{c|| c| c| c||c| c| c|}
\hline
\multicolumn{7}{|c|}{(f) Delicious-200k $D = 783k, K = 205k$}\\
\hline
\multicolumn{1}{|c||}{\multirow{1}{*}{Algorithm}}&P@1 & P@3 & P@5 & N@1 & N@3 & N@5 \\
\hline
\multicolumn{1}{|c||}{\multirow{1}{*}{LPSR}}&18.59&15.43  &14.07 &18.59 & 16.17& 15.13\\
\hline
\multicolumn{1}{|c||}{\multirow{1}{*}{PLT}}&45.37& 38.94 & 35.88 & -& -& -\\
\hline
\multicolumn{1}{|c||}{\multirow{1}{*}{GBDT-S}}&42.11& 39.06 & - & -& -& -\\
\hline
\multicolumn{1}{|c||}{\multirow{1}{*}{CRAFTML}}&\textbf{47.87}& \textbf{41.28} & 38.01& -& -& -\\
\hline
\multicolumn{1}{|c||}{\multirow{1}{*}{FastXML}}&43.07&  38.66& 36.19 &43.07 & 39.70& 37.83\\
\hline
\multicolumn{1}{|c||}{\multirow{1}{*}{PFastreXML}}&41.72&  37.83&35.58& 41.72 & 38.76& 37.08 \\
\hline
\hline
\multicolumn{1}{|c||}{\multirow{1}{*}{LdSM}}&45.26 &40.53 & \textbf{38.23} &\textbf{45.26} &\textbf{41.66} & \textbf{39.79}\\
\hline
\end{tabular}
% \setlength{\tabcolsep}{1.5pt}
% \begin{tabular}{c|| c| c| c||c| c| c|}
% \hline
% \multicolumn{7}{|c|}{(g) WikiLSHTC-325k $D = 1.6M, K = 325k$}\\
% \hline
% \multicolumn{1}{|c||}{\multirow{1}{*}{Algorithm}}&P@1 & P@3 & P@5 & N@1 & N@3 & N@5 \\
% \hline
% \multicolumn{1}{|c||}{\multirow{1}{*}{LPSR}} & 27.44&16.23  &11.77&27.44 & 23.04& 22.55\\
% \hline
% \multicolumn{1}{|c||}{\multirow{1}{*}{PLT}} & 45.67& 29.13 & 21.95 & -& -& -\\
% \hline
% \multicolumn{1}{|c||}{\multirow{1}{*}{GBDT-S}} & -& - & - & -& -& -\\
% \hline
% \multicolumn{1}{|c||}{\multirow{1}{*}{CRAFTML}} & \textbf{56.57}& 34.73 & 25.03& -& -& -\\
% \hline
% \multicolumn{1}{|c||}{\multirow{1}{*}{FastXML}} & 49.75&  33.10& 24.45 &49.75 & 45.23& 44.75	\\
% \hline
% \multicolumn{1}{|c||}{\multirow{1}{*}{PFastreXML}} & 56.05&  \textbf{36.79}&\textbf{27.09}& \textbf{56.05} & \textbf{50.59}& \textbf{50.13} \\
% \hline\hline
% \multicolumn{1}{|c||}{\multirow{1}{*}{LdSM}} & 55.00 &34.57 & 25.29 &55.00 &48.32&47.80\\
% \hline
% \end{tabular}
% \hfill

\vspace{0.125in}
\setlength{\tabcolsep}{1.5pt}
\begin{tabular}{c|| c| c| c||c| c| c|}
\hline
\multicolumn{7}{|c|}{(g) Amazon-670k $D = 135k, K = 670k$}\\
% N = 197k,
\hline
\multicolumn{1}{|c||}{\multirow{1}{*}{Algorithm}} &P@1 & P@3 & P@5 & N@1 & N@3 & N@5 \\
\hline
\multicolumn{1}{|c||}{\multirow{1}{*}{LPSR}} &28.65&24.88  &22.37 &28.65 & 26.40& 25.03\\
\hline
\multicolumn{1}{|c||}{\multirow{1}{*}{PLT}} &36.65& 32.12 & 28.85 & -& -& -\\
\hline
\multicolumn{1}{|c||}{\multirow{1}{*}{GBDT-S}} &-& - & - & -& -& -\\
\hline
\multicolumn{1}{|c||}{\multirow{1}{*}{CRAFTML}} &37.35& 33.31 & 30.62 & -& -& -\\
\hline
\multicolumn{1}{|c||}{\multirow{1}{*}{FastXML}} &36.99&  33.28& 30.53 &36.99 & 35.11& 33.86	\\
\hline
\multicolumn{1}{|c||}{\multirow{1}{*}{PFastreXML}} &39.46&  35.81&33.05& 39.46 & 37.78& 36.69 \\
\hline
\hline
\multicolumn{1}{|c||}{\multirow{1}{*}{LdSM}} &\textbf{42.63} & \textbf{38.09} & \textbf{34.70} &\textbf{42.63} &\textbf{40.37} & \textbf{38.89}\\
\hline
\end{tabular}
\end{table}

 \begin{table}[htp!]
 % \scriptsize
 \centering
 \setlength{\tabcolsep}{1.5pt}
 \caption{Precisions: $P@1$, $P@3$, and $P@5$ ($\%$) obtained by LdSM and OAA approach (Parabel) on common multi-label data sets.}
 \vspace{-0.2in}
 \centering
 \begin{tabular}{c|| c|c|c|| c|c| c|}
 \cline{2-7}
 \multicolumn{1}{c}{}&\multicolumn{3}{|c||}{\multirow{1}{*}{OAA (Parabel)}} &\multicolumn{3}{c|}{\multirow{1}{*}{LdSM}} \\
 \cline{2-7}
 \multirow{2}{*}{} & \rotatebox{0}{P@1}&\rotatebox{0}{P@3}&\rotatebox{0}{P@5} & \rotatebox{0}{P@1} & \rotatebox{0}{P@3} &  \rotatebox{0}{P@5}\\
 \hline
 \hline
 \multicolumn{1}{|c||}{\multirow{1}{*}{Mediamill}} &83.91&67.12& 52.99  & \textbf{90.64} & \textbf{73.60}& \textbf{58.62}\\
 \hline
 \multicolumn{1}{|c||}{\multirow{1}{*}{Bibtex}} &64.53&38.56& 27.94 & \textbf{64.69} & \textbf{39.70}&\textbf{29.25}\\
 \hline
 \multicolumn{1}{|c||}{\multirow{1}{*}{Delicious}} &67.44&61.83& 56.75 & \textbf{71.91} & \textbf{65.34}&\textbf{60.24}\\
 \hline
 % \multicolumn{1}{|c||}{\multirow{1}{*}{Eurlex-4k}} &NA&5.39& 3.65 & 5.43 & \textbf{0.03}\\
 % \hline
 \multicolumn{1}{|c||}{\multirow{1}{*}{AmazonCat-13k}} &93.03&\textbf{79.16}& \textbf{64.52} & \textbf{93.87 }& 75.41&57.86\\
 \hline
 \multicolumn{1}{|c||}{\multirow{1}{*}{Wiki10-31k}} &\textbf{84.31}&\textbf{72.57}& \textbf{63.39} & 83.74 & 71.74&61.51\\
 \hline
 \multicolumn{1}{|c||}{\multirow{1}{*}{Delicious-200k}} &\textbf{46.97}&40.08& 36.63 & 45.26 & \textbf{40.53}&\textbf{38.23}\\
 \hline
%  \multicolumn{1}{|c||}{\multirow{1}{*}{WikiLSHTC-325k}} &\underline{65.04}&\underline{43.23}& \underline{32.05} & 55.00 & 34.57&25.29\\
%  \hline
 \multicolumn{1}{|c||}{\multirow{1}{*}{Amazon-670k}} &\textbf{44.89}&\textbf{39.80}& \textbf{36.00} & 42.63 & 38.09&34.70\\
 \hline
 \end{tabular}
 \label{tab:pOAA}
%  \vspace{-0.1in}
 \end{table}
 
  \begin{table}[htp!]
 % \scriptsize
%  \right
 \setlength{\tabcolsep}{0.95pt}
 \caption{Propensity Score Precisions: $PSP@1$, $PSP@3$, and $PSP@5$ ($\%$) obtained by LdSM and OAA approach (Parabel) on common multi-label data sets.}
%  \vspace{-0.2in}
 \centering
 \begin{tabular}{c|| c|c|c|| c|c| c|}
 \cline{2-7}
 \multicolumn{1}{c}{}&\multicolumn{3}{|c||}{\multirow{1}{*}{OAA (Parabel)}} &\multicolumn{3}{c|}{\multirow{1}{*}{LdSM}} \\
 \cline{2-7}
 \multirow{3}{*}{} & \rotatebox{0}{PSP}&\rotatebox{0}{PSP}&\rotatebox{0}{PSP} & \rotatebox{0}{PSP} & \rotatebox{0}{PSP} &  \rotatebox{0}{PSP}\\
 &\rotatebox{0}{@1}&\rotatebox{0}{@3}&\rotatebox{0}{@5} & \rotatebox{0}{@1} & \rotatebox{0}{@3} &  \rotatebox{0}{@5}\\
 \hline
 \hline
 \multicolumn{1}{|c||}{\multirow{1}{*}{Mediamill}} &66.51&65.21& 64.30  & \textbf{70.27} & \textbf{69.66}& \textbf{68.86}\\
 \hline
 \multicolumn{1}{|c||}{\multirow{1}{*}{Bibtex}} &50.88&52.42& 57.36 & \textbf{52.01} & \textbf{54.38}&\textbf{60.34}\\
 \hline
 \multicolumn{1}{|c||}{\multirow{1}{*}{Delicious}} &32.69&34.00& 34.53 & \textbf{37.27} & \textbf{38.32}&\textbf{38.46}\\
 \hline
 % \multicolumn{1}{|c||}{\multirow{1}{*}{Eurlex-4k}} &NA&5.39& 3.65 & 5.43 & \textbf{0.03}\\
 % \hline
 \multicolumn{1}{|c||}{\multirow{1}{*}{AmazonCat-13k}} &50.93&\textbf{64.00}&\textbf{72.08} & \textbf{51.06 }& 58.67&60.47\\
 \hline
 \multicolumn{1}{|c||}{\multirow{1}{*}{Wiki10-31k}} &11.66&\textbf{12.73}& \textbf{13.68} & \textbf{11.87}& 12.35&12.89\\
 \hline
 \multicolumn{1}{|c||}{\multirow{1}{*}{Delicious-200k}} &\textbf{7.25}&7.94& 8.52& 7.16 & \textbf{8.26}&\textbf{9.11}\\
 \hline
 \multicolumn{1}{|c||}{\multirow{1}{*}{Amazon-670k}} &25.43&29.43& 32.85 & \textbf{28.14} & \textbf{30.82}&\textbf{33.16}\\
 \hline
%   \multicolumn{1}{|c||}{\multirow{1}{*}{(8)}} &26.76&33.27& \textbf{37.36} & \textbf{30.88}& \textbf{33.65}&34.59\\
%  \hline
 \end{tabular}
 \label{tab:pOAAPS}
%  \vspace{-0.1in}
 \end{table}

high  multi-label accuracy with logarithmic-depth trees. This new method is therefore suitable for applications involving large label spaces. 

\newpage
\clearpage
\begin{small}
\bibliography{aistats2020}
\end{small}

\newpage
\clearpage
\onecolumn

\hrule height 4pt
  \vskip 0.25in
  \vskip -\parskip%
\begin{center}
{\LARGE\bf Logarithm-depth Streaming Multi-label Decision Trees\\ (Supplementary material) \par} 
\end{center}
\vskip 0.29in
  \vskip -\parskip
  \hrule height 1pt
  \vskip 0.09in%

\begin{center}
\begin{large}
{\bf Abstract}  \\
\end{large}
\vskip 0.5ex
\end{center}

\begin{quote}
This Supplement presents additional details in support of the full article. These include the proofs of the theoretical statements from the main body of the paper and additional theoretical results. We also provide additional algorithm's pseudo-codes. The Supplement also contains the description of the experimental setup, and additional experiments and figures to provide further empirical support for the proposed methodology.
\end{quote}

\vskip 1ex

% \section{\textcolor{red}{removing: OBJECTIVE FUNCTION: BINARY CASE}}

% In the binary case the objective simplifies to the following form:
% \vspace{-0.1in}
% \begin{equation}
% J \coloneqq \underbrace{\left| P_R - P_L\right|}_{\text{balancing term}} - \underbrace{\underbrace{\lambda_1 \sum_{i=1}^k \pi_i \left|P_R^i - P_L^i\right|}_{\text{class integrity term}} + \underbrace{\lambda_2\left|P_R + P_L-1\right|}_{\text{multi-way penalty}}}_{\text{purity term}},
% \label{eq:Bobj}
% \end{equation}
% \vspace{-0.1in}

% where $P_R$ and $P_R^i$ ($P_L$ and $P_L^i$) denote the probabilities that the example reaches right (left) child, marginally and conditional on class $i$ respectively.

% In case of the binary tree, each node is equipped with two linear classifiers, $h_R$ and $h_L$.
\section{ADDITIONAL THEORETICAL RESULTS}
Next lemma shows that in isolation, when the purity of the split is perfect, decreasing the value of the objective leads to recovering more balanced splits.

\begin{lemma} If a node split is perfectly pure, then
\begin{equation}
\beta \leq J - J^*.
\end{equation}
\label{lem:balance}
\vspace{-0.1in}
\end{lemma}

Next lemma shows that in isolation, when the balancedness of the split is perfect, decreasing the value of the objective leads to recovering more pure splits.

\begin{lemma}
If a node split is perfectly balanced and assuming that the following condition holds: $\lambda_1(M-1) \geq \lambda_2 \geq \lambda_1\frac{M-1}{2}$,  then 
\vspace{-0.1in}
\begin{equation}
\alpha \leq (J +\lambda_2)\frac{2}{M(2\lambda_2-\lambda_1(M-1))}.
\end{equation}
\label{lem:pure}
\vspace{-0.1in}
\end{lemma}

Below we provide a new assumption and corresponding theorem that generalizes Theorem~\ref{thm:errorbound}, by removing the balancedness assumption.
\begin{assumption}
$\gamma$-Weak Hypothesis Assumption: for any distribution $\mathcal{P}$ over the data, at each node of the tree $\mathcal{T}$ there exist a partition such that
$\sum_i \pi_i \left|\frac{P_R^i}{P_R} - \frac{P_L^i}{P_L}\right| \geq \gamma$, where $\gamma \in (0,1]$.
\label{as:WHA4}
\end{assumption}
\begin{theorem}
Under the Weak Hypothesis Assumptions~\ref{as:WHA4} and~\ref{as:WHA2} for any $\alpha \in [0,1]$ to obtain $e_r(\mathcal{T}) \leq \alpha$ it suffices to have a tree with $t$ internal nodes that satisfy
$
(t+1)\geq (\frac{1}{\alpha})^{\frac{16 \ln K}{c r^2 \gamma^2 (1- b) \log_2(e)}}$, where $b = |P_R + P_L - 1|$.
\label{thm:errorwobalance}
\end{theorem}

Below we consider the weak hypothesis assumption that generalizes the Assumption~\ref{as:WHA1} to the $M$-ary case and prove corresponding lemma that generalizes Lemma ~\ref{lem:error2}.
\begin{assumption}[Generalization of Assumption~\ref{as:WHA1}]
$\gamma$-Weak Hypothesis Assumption: for any distribution $\mathcal{P}$ over the data, at each node $n$ of the tree $\mathcal{T}$ there exist a partition such that $\sum_{i=1}^{K}\sum_{j=1}^{M}\sum_{l=1}^{M}\pi_i\left|P_j^i - P_l^i\right| \geq \gamma$, where $\gamma \in (0,1]$.
\label{as:WHA3}
\end{assumption}
\begin{lemma} [Generalization of Lemma ~\ref{lem:error2}] Under the Weak Hypothesis Assumption~\ref{as:WHA3}, the $e_r(\mathcal{T})$ is monotonically decreasing with every split of the tree.
\label{lem:error3}
\end{lemma}

\subsection{Relation of the Objective to Shannon Entropy and Error Bound (Binary Tree Case)}

In this section we first show the relation of the objective $J$ to a classical decision-tree criterion, Shannon entropy, and specifically we demonstrate that minimizing the objective leads to the reduction of this criterion. We restrict ourselves to the case of binary tree. We omit the analysis for the $M$-ary to avoid over-complicating the notation. The entropy of tree leaves in the case when examples can be sent to multiple directions can be calculated as:
\vspace{-0.05in}
\begin{equation}
G = \sum_{\tilde{\mathcal{L}} \subset \mathcal{L}} w_{\tilde{\mathcal{L}}} \sum_{i=1}^K \rho_i^{\tilde{\mathcal{L}}} \ln(\frac{1}{\rho_i^{\tilde{\mathcal{L}}}})
\end{equation},
\vspace{-0.2in}

where $\mathcal{L}$ is the set of all tree leaves, $\tilde{\mathcal{L}}$ is a subset of the leaves (the summation is taken over all the possible subsets), $\rho_i^{\tilde{\mathcal{L}}}$ is the probability that example with label $i$ reaches all the leaves in $\tilde{\mathcal{L}}$, and $w_{\tilde{\mathcal{L}}}$ is the weight of subset of leaves. This weight is defined as the probability that a randomly chosen point from distribution $\mathcal{P}$ reaches all leaves in $\tilde{\mathcal{L}}$. Also note that $\sum_{\tilde{\mathcal{L}} \subset \mathcal{L}} w_{\tilde{\mathcal{L}}} = 1$ and $w_{\tilde{\mathcal{L}}=\emptyset} = 0$. \\
\begin{theorem}
Under the Weak Hypothesis Assumptions \ref{as:WHA1} and \ref{as:WHA2}, and an additional assumption that each node produces perfectly balanced split, for any $\kappa \in [0,\ln K]$ to obtain $G^e_t \leq \kappa$ it suffices to have a tree with $t$ internal nodes that satisfy
\[
(t+1)\geq (\frac{G_1}{\kappa})^{\frac{16 \ln K}{c r^2 \gamma^2 (1- b) \log_2(e)}},
\]
where $b = |P_R + P_L - 1|$.
\label{thm:entropy}
\end{theorem}

\section{THEORETICAL PROOFS}

% \begin{lemma}(Binary tree)
% For any hypotheses $h_R,h_L\in\mathcal{H}$, the objective $J$ defined in Equation~\ref{eq:Bobj} satisfies $J\in[-\lambda_1,\;\lambda_2]$ and it is minimized if and only if the split is perfectly balanced and perfectly pure. 
% \label{lem:binary}
% \end{lemma}

\begin{proof}[Proof of Lemma~\ref{lem:binary}]
We rewrite the objective using the total law of probability:
\begin{equation}
J = \left| \sum_{i=1}^K \pi_i (P_R^i - P_L^i)\right| - \lambda_1 \sum_{i=1}^K \pi_i \left|P_R^i - P_L^i\right| + \lambda_2\left| \sum_{i=1}^K \pi_i (P_R^i + P_L^i) - 1\right|,
\end{equation}
where $P_R^i, P_L^i \in [0,1]$ for all $i = 1,2,\dots,K$. The objective admits optimum on the extremes of the $[0,1]$ interval. Therefore, we define the following:
\begin{equation}
  L_1 = \{ i: i \in \{1,\dots,K\}, P_R^i=1 \;\&\; P_L^i=1 \}, \;\;L_2 = \{ i: i \in \{1,\dots,K\}, P_R^i=0 \;\&\; P_L^i=0 \}, 
\end{equation}
\begin{equation}
L_3 = \{ i: i \in \{1,\dots,K\}, P_R^i=1 \;\&\; P_L^i=0 \}, \;\;L_4 = \{ i: i \in \{1,\dots,K\}, P_R^i=0 \;\&\; P_L^i=1 \}
\end{equation}
By substituting the above in the objective we have:
\begin{equation}
  J = \left| \sum_{i\in L_3} \pi_i - \sum_{i\in L_4} \pi_i\right| - \lambda_1 \sum_{i\in(L_3\cup L_4)} \pi_i + \lambda_2\left| \sum_{i\in(L_3\cup L_4)} \pi_i + \sum_{i\in L_1} 2\pi_i - 1\right|.  
\end{equation}
We send each example either to the right, left or both directions:
\begin{equation}
   \sum_{i\in(L_1 \cup L_3\cup L_4)} \pi_i = \sum_{i\in L_1} \pi_i + \sum_{i\in L_3} \pi_i  + \sum_{i\in L_4} \pi_i  = 1. 
\end{equation}
Thus we can further write
\begin{equation}
 J = \left| 1 - \sum_{i\in L_1} \pi_i - 2\sum_{i\in L_4} \pi_i\right| - \lambda_1 (1 - \sum_{i\in L_1} \pi_i) + \lambda_2 \sum_{i\in L_1} \pi_i.   
\end{equation}
For ease of notation, we define $a \coloneqq \sum_{i\in L_4} \pi_i$,  $a' \coloneqq \sum_{i\in L_3} \pi_i$, and $b \coloneqq \sum_{i\in L_1} \pi_i$. Therefore
\begin{equation}
  J = \left| 1 - b - 2a\right| - \lambda_1 (1 - b) + \lambda_2 b = \left| b + 2a' - 1\right| - \lambda_1 (1 - b) + \lambda_2 b,  
\end{equation}
where $a,b \in [0,1]$. Since we are interested in bounding $J$, we consider the values of $a$ and $b$ at the extremes of $[0,1]$ interval:
\begin{equation}
    \text{if}\; a = 1 \; \text{then} \; b=0 \;\rightarrow\; J = 1-\lambda_1, \;\;\;\;\;\;\; \text{if}\; b = 1 \; \text{then} \; a=0 \;\rightarrow\; J = \lambda_2
\end{equation}
\begin{equation}
    \text{if}\; a = 0 \; \text{then} \;\Big\{\begin{array}{lr}
        b =0 \; (a'=1) \; \rightarrow & J = 1-\lambda_1\\
        b =1 \;\hspace{0.51in}\rightarrow & J = \lambda_2\\
         \end{array}
         \end{equation}
         \begin{equation}
         \text{if}\; b = 0 \; \text{then} \;\Bigg\{\begin{array}{lr}
        a =0 \; (a'=1) \;\rightarrow & J = 1-\lambda_1\\
        a =1 \;\hspace{0.51in}\rightarrow & J = 1-\lambda_1\\
        a =0.5 \; \hspace{0.41in}\rightarrow & J = -\lambda_1\\
         \end{array}
\end{equation}

Therefore $J \in [-\lambda_1,\;\lambda_2]$.\\ Next, we show that the perfectly balanced and pure split is attained at the minimum of the objective. The perfectly balanced split is achieved when $P_R=P_L$ and then the balancing term in the objective becomes zero. The perfectly pure split is achieved when the class integrity term in the objective satisfies $\sum_{i=1}^K \pi_i \left|P_R^i - P_L^i\right| = \sum_{i=1}^K \pi_i =1 $. Simultaneously, the following holds $\sum_{i=1}^K \pi_i (P_R^i + P_L^i) = 1$, and therefore the multi-way penalty is zero as well. Thus, $J=0-\lambda_1+0=-\lambda_1$. In order to prove the opposite direction of the claim, recall that the minimum of the objective occurs for $b=0$ and $a=0.5$. Since $a+a'+b=1$, therefore $a'=0.5$. This corresponds to the perfectly pure and balanced split.
\end{proof}

\begin{proof}[Proof of Lemma~\ref{lem:mary}]
$P_j^i \in [0,\;1]$  for all $i = 1,2,\dots,K$ and $j = 1,2,\dots,M$. The objective admits optimum on the extremes of the $[0,1]$ interval. In the following proof we consider a different approach than in the proof of Lemma~\ref{lem:binary}. In order to get the minimum of the objective, we try to minimize each of its terms separately and on the top of that incorporate their correlations. For now, we assume that the first term, the balancing term, is minimized and therefore is equal to zero. We define case $C_n$ as the scenario when for any $i = 1,2,\dots,K$, $P_j^i=1$ for $n$ ``directions'' ( $n \leq M$), i.e. $n$ distinct $j$s such that $j \in \{1,\;2,\dots,\;M\}$,  and $P_j^i=0$ for the remaining $j$'s. The class integrity and multi-way penalty terms can then be derived as follows:
\begin{equation}
 J_{\text{class integrity term}|C_n} = \lambda_1\sum_{i=1}^K \sum_{j = 1}^M\sum_{l=j+1}^M \pi_i\left|P_j^i - P_l^i\right| = n(M-n),   
\end{equation}

\begin{equation}
  J_{\text{multi-way penalty term}|C_n} = \lambda_2\left(\sum_{j=1}^M P_j\right)-1 = n - 1.  
\end{equation}

Therefore, the objective value would then become: $J=-\lambda_1 n(M-n) + \lambda_2(n-1)$. We aim to have the minimum of the objective for perfectly pure split. The perfectly pure split is achieved when case $C_1$ holds. Therefore, we need:
\begin{equation}
 -\lambda_1(M-1) < -\lambda_1 n(M-n) + \lambda_2(n-1) \;\;\;\text{for}\; n \in \{2,\dots,M \}.
\end{equation}

The lower-bound of the right side is achieved for $n=2$:
\begin{equation}
 -\lambda_1(M-1) < -\lambda_1 2(M-2) + \lambda_2 \;\;\;\rightarrow \;\;\;M-3<\frac{\lambda_2}{\lambda_1}.   
\end{equation}
With the above condition, the minimum of the objective is equal to $-\lambda_1(M-1)$. Note that our first assumption on the balancing term can still hold for all $C_n$ cases. Therefore, we have shown that the minimum of the objective corresponds to the perfectly pure and balanced split. \\
In order to get the upper-bound for $J$, we first show that $J_{\text{balancing term}}\leq J_{\text{class integrity term}}$ as follows:
\begin{eqnarray}
 J_{\text{balancing term}} = \sum_{j = 1}^M\sum_{l=j+1}^M\left|P_j - P_l\right| =  \sum_{j = 1}^M\sum_{l=j+1}^M\left|\sum_{i=1}^K\pi_i(P_j^i-P_l^i)\right|\\
 \leq  \sum_{j = 1}^M\sum_{l=j+1}^M \sum_{i=1}^K \pi_i\left|P_j^i - P_l^i\right| = J_{\text{class integrity term}}.
\end{eqnarray}
Therefore, the maximum of the summation of the terms is achieved when $J_{\text{balancing term}}=J_{\text{class integrity term}}$. The maximum of the multi-way penalty term is attained when sending all examples to every direction, resulting in $J_{\text{multi-way penalty term}}=(M-1)$. In this case, $J_{\text{balancing term}}=J_{\text{class integrity term}} = 0$, and thus, $J = \lambda_2(M-1)$. Hence, we have $J\in[-\lambda_1(M-1),\;\lambda_2(M-1)]$.
\end{proof}

\begin{proof}[Proof of Lemma~\ref{lem:balance}]
The perfectly pure split is attained when $P_j^i=1$ for only one value of $j$, and $P_j^i=0$ for the remaining $j$'s. This leads the class integrity term to satisfy $\sum_{j = 1}^M\sum_{l=j+1}^M \sum_{i=1}^K \pi_i\left|P_j^i - P_l^i\right| = (M-1)$ and the multi-way penalty term to satisfy $\sum_{i=1}^k \pi_i \sum_{j = 1}^M P_j^i - 1 =0$. Thus we have:
\begin{eqnarray}
J -J^* &=& \sum_{j = 1}^M\sum_{l=j+1}^M\left|P_j - P_l\right| \\
&=& \sum_{j = 1}^M\sum_{l=j+1}^M\left|\left(P_j - \frac{\sum_{i=1}^M P_i}{M}\right)- \left(P_l - \frac{\sum_{i=1}^M P_i}{M}\right)\right|.
\end{eqnarray}
Let $j^*=\text{argmax}_{j \in \{1,2,\dots,M\}}|P_j - \frac{\sum_{i=1}^M P_i}{M}|$. Without loss of generality assume $P_{j^*} - \frac{\sum_{i=1}^M P_i}{M}\geq 0$ and in that case there exists an $l^*$ such that $P_{l^*} - \frac{\sum_{i=1}^M P_i}{M}\leq 0$. Therefore we have:
\begin{eqnarray}
J -J^* &\geq& \left|\left(P_{j^*} - \frac{\sum_{i=1}^M P_i}{M}\right)- \left(P_{l^*} - \frac{\sum_{i=1}^M P_i}{M}\right)\right| \\
&\geq& \left|(P_{j^*} - \frac{\sum_{i=1}^M p_i}{M})\right| = \beta.
\end{eqnarray}
\end{proof}

\begin{proof}[Proof of Lemma~\ref{lem:balancefix}]
Consider a split with a fixed purity factor $\alpha$. $J_{\text{purity}}^{\alpha}$ denotes the sum of the class integrity and multi-way penalty terms of the objective function. When subtracting them from the total value of the objective at node $n$ we obtain the balancing term. Thus we have:
\begin{eqnarray}
J -J_{\text{purity}}^{\alpha} &=& \sum_{j = 1}^M\sum_{l=j+1}^M\left|P_j - P_l\right| \\
&=& \sum_{j = 1}^M\sum_{l=j+1}^M\left|\left(P_j - \frac{\sum_{i=1}^M P_i}{M}\right)- \left(P_l - \frac{\sum_{i=1}^M P_i}{M}\right)\right|.
\end{eqnarray}
Let $j^*=\text{argmax}_{j \in \{1,2,\dots,M\}}|P_j - \frac{\sum_{i=1}^M P_i}{M}|$. Without loss of generality assume $P_{j^*} - \frac{\sum_{i=1}^M P_i}{M}\geq 0$ and in that case there exists an $l^*$ such that $P_{l^*} - \frac{\sum_{i=1}^M P_i}{M}\leq 0$. Therefore we have:
\begin{eqnarray}
J -J_{\text{purity}}^{\alpha} &\geq& \left|\left(P_{j^*} - \frac{\sum_{i=1}^M P_i}{M}\right)- \left(P_{l^*} - \frac{\sum_{i=1}^M P_i}{M}\right)\right| \\
&\geq& \left|(P_{j^*} - \frac{\sum_{i=1}^M p_i}{M})\right| = \beta.
\end{eqnarray}
\end{proof}

\begin{proof}[Proof of Lemma~\ref{lem:pure}]
The perfectly balanced split is attained when $P_1=P_2=...=P_M$. This zeros out the balancing term in the objective function. Hence:
\begin{eqnarray}
J =  -\lambda_1 \sum_{i=1}^K \sum_{j = 1}^M\sum_{l=j+1}^M \pi_i\left|P_j^i - P_l^i\right| + \lambda_2 \left(\sum_{j=1}^MP_j - 1\right)
\\
=  -\lambda_1 \sum_{i=1}^K \sum_{j = 1}^M\sum_{l=j+1}^M \pi_i\left|P_j^i - P_l^i\right| + \lambda_2 \left(\sum_{i=1}^K \sum_{j = 1}^M \pi_i P_j^i-1\right)
\\
\geq  -\lambda_1\frac{M-1}{2} \sum_{i=1}^K \sum_{j = 1}^M \pi_i P_j^i + \lambda_2 \left(\sum_{i=1}^K \sum_{j = 1}^M \pi_i P_j^i-1\right).
\end{eqnarray}

Thus we have:
\begin{eqnarray}
J+ \lambda_2 \geq \left(\lambda_2-\lambda_1\frac{M-1}{2}\right) \sum_{i=1}^K \sum_{j = 1}^M \pi_i P_j^i 
\\
\geq \left(\lambda_2-\lambda_1\frac{M-1}{2}\right) \sum_{i=1}^K \sum_{j = 1}^M \pi_i \min(P_j^i,\sum_{l=1}^M P_l^i-P_j^i)
\\
\geq \left(\lambda_2-\lambda_1\frac{M-1}{2}\right) M\alpha.
\end{eqnarray}
\end{proof}

\begin{proof}[Proof of Lemma~\ref{lem:purefix}]
Consider a split with a fixed balancedness factor $\beta$. $J_{\text{balance}}^{\beta}$ denotes the balancing term of the objective function. When subtracting it from the total value of the objective at node $n$ we will obtain the sum of the class integrity and multi-way penalty terms. Hence:
\begin{eqnarray}
J - J_{\text{balance}}^{\beta} =  -\lambda_1 \sum_{i=1}^K \sum_{j = 1}^M\sum_{l=j+1}^M \pi_i\left|P_j^i - P_l^i\right| + \lambda_2 \left(\sum_{j=1}^MP_j - 1\right)
\\
=  -\lambda_1 \sum_{i=1}^K \sum_{j = 1}^M\sum_{l=j+1}^M \pi_i\left|P_j^i - P_l^i\right| + \lambda_2 \left(\sum_{i=1}^K \sum_{j = 1}^M \pi_i P_j^i-1\right)
\\
\geq  -\lambda_1\frac{M-1}{2} \sum_{i=1}^K \sum_{j = 1}^M \pi_i P_j^i + \lambda_2 \left(\sum_{i=1}^K \sum_{j = 1}^M \pi_i P_j^i-1\right).
\end{eqnarray}

Thus we have:
\begin{eqnarray}
J - J_{\text{balance}}^{\beta} + \lambda_2 \geq \left(\lambda_2-\lambda_1\frac{M-1}{2}\right) \sum_{i=1}^K \sum_{j = 1}^M \pi_i P_j^i 
\\
\geq \left(\lambda_2-\lambda_1\frac{M-1}{2}\right) \sum_{i=1}^K \sum_{j = 1}^M \pi_i \min(P_j^i,\sum_{l=1}^M P_l^i-P_j^i)
\\
\geq \left(\lambda_2-\lambda_1\frac{M-1}{2}\right) M\alpha.
\end{eqnarray}
\end{proof}

\begin{proof}[Proof of Theorem~\ref{thm:entropy}]
In our algorithm, we recursively find the leaf node with the heaviest weight and decide to partition it to two children. Suppose, after $t$ splits the leaf node $n$ has the highest weight, namely $w_n$, which will be denoted with $w$ for brevity. This weight is defined as the probability that a randomly chosen data point $x$ drawn from a fixed distribution $\mathcal{P}$ reaches the leaf. Let $w_{R\; \text{only}}$ and $w_{L\;\text{only}}$ be the weight of examples reaching only to the right and left child of node $n$, and $w_{{both}}$ be the weight of examples reaching to both children. Also let $P_{{both}} = |P_R + P_L - 1|$. Note that $w_{R\; \text{only}} = w P_{R\; \text{only}} = w(P_R - P_{{both}})$ and $w_{L\; \text{only}} = w P_{L\; \text{only}} = w(P_L - P_{{both}})$. Let $\pmb{\rho}$ be a vector with K elements, which its $i^{th}$ element is $\rho_i$. Furthermore, let  $\pmb{\rho}_R$, and $\pmb{\rho}_L$ be K-element vectors with  $\rho_{i,R}$ and $\rho_{i,L}$ at its $i^{th}$ entry. Note that $\rho_{i,R} = \frac{\rho_i P_R^i}{P_R}$, and $\rho_{i,L}= \frac{\rho_i P_L^i}{P_L}$. Before the node partition the contribution of node $n$ to the total entropy-based objective is $w\tilde{G}(\pmb{\rho})$. After the split this contribution will be $w_{R\; \text{only}}\tilde{G}(\pmb{\rho}_R)+w_{L\; \text{only}}\tilde{G}(\pmb{\rho}_L)+w_{{both}}\tilde{G}(\pmb{\rho})$ (Note that for the examples being sent to both directions we average the histograms of the left and right children. Also note that $(w_{R\; \text{only}} + w_{R\; \text{only}} + w_{{both}}) =1 $) Therefore, we have:
\begin{eqnarray}
\Delta_t := G_t - G_{t+1} = w [\tilde{G}(\pmb{\rho}) - P_{R\; \text{only}}\tilde{G}(\pmb{\rho}_R)-P_{L\; \text{only}}\tilde{G}(\pmb{\rho}_L)-P_{\text{both}}\tilde{G}(\pmb{\rho})]
\\
= w [\tilde{G}(\pmb{\rho}) - (P_R - P_{both})\tilde{G}(\pmb{\rho}_R) - (P_L - P_{both})\tilde{G}(\pmb{\rho}_L) - P_{both}\tilde{G}(\pmb{\rho})].
\end{eqnarray}
Recall that the Shannon entropy is strongly concave with respect to $l_1$-norm~\citep[see][Example 2.5]{ShaiSS2012}, and $\pmb{\rho} = (P_R - \frac{1}{2} P_{both})\pmb{\rho}_R + (P_L - \frac{1}{2} P_{both})\pmb{\rho}_L $, where $P_{both} = P_R + P_L - 1$. Without loss of generality assume $P_R = P_L + \eta$. Hence we re-write $\Delta_t$ as follows:

\begin{eqnarray}
\Delta_t = w [(1 - P_{both})\tilde{G}(\pmb{\rho}) - (\frac{1 + \eta - P_{both}}{2})\tilde{G}(\pmb{\rho}_R) - (\frac{1 - \eta - P_{both}}{2})\tilde{G}(\pmb{\rho}_L)] 
\\
= w (1 - P_{both})[\tilde{G}(\pmb{\rho}) - (\frac{1 + \eta - P_{both}}{2(1 - P_{both})})\tilde{G}(\pmb{\rho}_R) - (\frac{1 - \eta - P_{both}}{2(1 - P_{both})})\tilde{G}(\pmb{\rho}_L)]. \\
\end{eqnarray}
We can then use the result from Theorem 2.1.9 in~\cite{nestrov}:
\begin{eqnarray}
\Delta_t \geq w(1 - P_{both})\left[ \frac{1}{8}||\pmb{\rho}_R-\pmb{\rho}_L||_1^2 \right]
\\
= w(1 - P_{both})r^2\left[ \frac{1}{8}||\pmb{\pi}_R-\pmb{\pi}_L||_1^2 \right]
\end{eqnarray}
\begin{eqnarray}
= w(1 - P_{both})r^2\left[ \frac{1}{8}\left(\sum_{i=1}^K\left|\frac{\pi_i P_R^i}{P_R}-\frac{\rho_i P_L^i}{P_L}\right|\right)^2\right].
\label{eqn:52}
\end{eqnarray}
Here we use the assumption that we have a balance split, i.e. $P_R = P_L$, therefore we continue as follows:
\begin{eqnarray}
= w(1 - P_{both}) \frac{r^2}{8 P_R^2}\left(\sum_{i=1}^K \pi_i |P_R^i- P_L^i| \right)^2
\end{eqnarray}
\begin{eqnarray}
\geq w(1 - P_{both}) \frac{r^2}{8}\left(\sum_{i=1}^K \pi_i |P_R^i- P_L^i| \right)^2.
\end{eqnarray}
Now by applying the WHA~\ref{as:WHA2}:
\begin{eqnarray}
\Delta_t \geq w(1 - b) \frac{r^2}{8}\gamma^2.
\label{eqn:55}
\end{eqnarray}
Note that by WHA~\ref{as:WHA2} $b \in [0,1)$. Also note that $w\geq \frac{G_t c}{(t+1)\ln K}$. This comes from the fact that at each step we choose the leaf node with maximum weight. Hence with WHA2, $w = \max_{l \in \mathcal{L}} w_l \geq \frac{c}{(t+1)}$. Also note that uniform distribution maximizes the entropy, i.e. $G_t \leq \ln K$. Accordingly we have: 
\begin{eqnarray}
\Delta_t \geq \frac{G_t c}{(t+1)\ln K}[ \frac{r^2}{8} \gamma^2 (1-b)].
\end{eqnarray}
By letting $\eta =\frac{1}{2} \sqrt{\frac{c r^2\gamma^2 (1-b)}{2\ln K}}$, we have $\Delta_t \geq \frac{\eta^2G_t}{(t+1)}$. Thus, we have the following recursion inequality:
\begin{eqnarray}
G_{t+1} \leq G_t - \Delta_t \leq G_t -\frac{\eta^2 G_t}{(t+1)} = G_t[1-\frac{\eta^2}{(t+1)}].
\end{eqnarray}
Then by applying the same proof technique as in~\citet{boost99} we get the following relationship:
\begin{eqnarray}
G_{t+1} \leq G_1 e^{-\eta^2 \log_2(t+1)/2}.
\end{eqnarray}
Therefore, to reduce $G_{t+1} \leq \kappa$ it suffices to have (t+1) splits such that $\log_2(t+1) \geq \ln(\frac{G_1}{\kappa})^{\frac{2}{\eta^2}}$. Substituting $\log_2(t+1) = \ln(t+1)\log_2(e)$ results in:
\begin{eqnarray}
\ln(t+1) \geq \ln(\frac{G_1}{\kappa})^{\frac{2}{\eta^2 \log_2(e)}} \Leftrightarrow (t+1)\geq (\frac{G_1}{\kappa})^{\frac{2}{\eta^2 \log_2(e)}}.
\end{eqnarray}
\end{proof}

We next proceed to the proof of Theorem~\ref{thm:errorbound}.

\begin{proof}[Proof of Theorem~\ref{thm:errorbound}]
This proof follows the proof of the Theorem~\ref{thm:entropy}. Below we directly calculate the error bound. Recall $w_{\tilde{\mathcal{L}}}$ to be the probability that a data point x reached
the subset of leaves $\tilde{\mathcal{L}}$. Recall that $\rho_i^{\tilde{\mathcal{L}}}$ is the probability that the data point $x$ has label $i$ given that $x$ reached $\tilde{\mathcal{L}}$, i.e. $\rho_i^{\tilde{\mathcal{L}}} = P(i \in t(x) | x \text{ reached } \tilde{\mathcal{L}})$. Note that each example has $r$ labels, and let's assume we assign first majority $r$ labels from the $\rho_i^{\tilde{\mathcal{L}}}$ histogram to any example reaching $\tilde{\mathcal{L}}$,
i.e. $y_r(x)=\{j_1,j_2,...,j_r\}$, where $j_1=\text{argmax}_{k \in \{1,2,\dots,K\}} (\rho^{\tilde{\mathcal{L}}}_k)$, $j_2=\text{argmax}_{k \in \{1,2,\dots,K\}\setminus j_1} (\rho^{\tilde{\mathcal{L}}}_k)$,..., $j_r=\text{argmax}_{k \in \{1,2,\dots,K\}\setminus \{j_1,...,j_{r-1}\}} (\rho^{\tilde{\mathcal{L}}}_k)$. We then expand the $r$-level multi-label error as follows:
\begin{eqnarray}
\epsilon_r(\mathcal{T}) \!\!\!\!\!&=&\!\!\!\!\! \frac{1}{r}\sum_{i=1}^KP(i \in y_r(x),i \notin t(x))\\
\!\!\!\!\!&=&\!\!\!\!\! \frac{1}{r}\sum_{i=1}^KP(i \in t(x),i \notin y_r(x))\\
&=&\!\!\!\!\! \frac{1}{r}\sum_{{\tilde{\mathcal{L}}} \in \mathcal{L}}\!w_{\tilde{\mathcal{L}}}\!\sum_{i=1}^K\!P(i \!\in\! t(x),i \!\notin\! y_r(x)|x \;\text{reached}\; {\tilde{\mathcal{L}}}) \\
&=&\!\!\!\!\! \frac{1}{r}\sum_{{\tilde{\mathcal{L}}} \in \mathcal{L}}w_{\tilde{\mathcal{L}}}\sum_{\substack{i = 1\\i\neq j_1,...,j_R}}^K\!\!\!P(i \in t(x)|x \;\text{reached}\; {\tilde{\mathcal{L}}}) \\
&=&\!\!\!\!\! \frac{1}{r}\sum_{{\tilde{\mathcal{L}}} \in \mathcal{L}}w_{\tilde{\mathcal{L}}}\left(\sum_{i=1}^K \rho_i^{\tilde{\mathcal{L}}} - \max_{k\in\{1,2,\dots,K\}}\rho_k^{\tilde{\mathcal{L}}}\right. \label{eqn:last} -\max_{k\in\{1,2,\dots,K\}\setminus j_1}\rho_k^{\tilde{\mathcal{L}}}\\
&&\:\:\:\:\:\:\:\:\:\:\:\:\:\:\:\:\:\:\:\:\:\:\left. -\max_{k\in\{1,2,\dots,K\}\setminus \{j_1,j_2\}}\rho_k^{\tilde{\mathcal{L}}} -\dots-\max_{k\in\{1,2,\dots,K\}\setminus \{j_1,j_2,\dots,j_{r-1}\}}\rho_k^{\tilde{\mathcal{L}}}\right),\nonumber
\end{eqnarray}

where $w_{\tilde{\mathcal{L}}}$ denote the probability that example $x$ reaches ${\tilde{\mathcal{L}}}$ and $\mathcal{L}$ denote the set of all leaves of the tree. 

Next we will find the Shannon entropy bound with respect to the error and show that the entropy of the tree, denoted as $G(\mathcal{T})$, upper-bounds the error. Note that:
\begin{eqnarray}
G(\mathcal{T}) = \sum_{\tilde{\mathcal{L}} \in \mathcal{L}} w_{\tilde{\mathcal{L}}} \sum_{i=1}^K \rho_i^{\tilde{\mathcal{L}}} \ln\left(\frac{1}{\rho_i^{\tilde{\mathcal{L}}}}\right)
&\geq& \sum_{l \in \mathcal{L}} w_{\tilde{\mathcal{L}}} \sum_{\substack{i = 1\\i\neq j_1,...,j_r}}^K \rho_i^{\tilde{\mathcal{L}}} \ln\left(\frac{1}{\rho_i^{\tilde{\mathcal{L}}}}\right).
\end{eqnarray}
Note that $\sum_{i = 1}^K\rho_i^{\tilde{\mathcal{L}}} = r$. Thus for any $i = 1,2,\dots,K$ such that $i\neq j_1,...,j_r$ it must hold that $\rho_i^{\tilde{\mathcal{L}}} \leq \frac{1}{2}$. We continue as follows

\begin{eqnarray}
G(\mathcal{T}) \!\!\!\!\!&\geq&\!\!\!\!\! \sum_{{\tilde{\mathcal{L}}} \in \mathcal{L}} w_{\tilde{\mathcal{L}}} \sum_{\substack{i = 1\\i\neq j_1,...,j_r}}^K \rho_i^{\tilde{\mathcal{L}}} \ln(2)\\
&\geq&\!\!\!\!\! \ln(2)\sum_{{\tilde{\mathcal{L}}} \in \mathcal{L}}w_{\tilde{\mathcal{L}}}\left(\sum_{i=1}^K \rho_i^{\tilde{\mathcal{L}}} - \max_{k\in\{1,2,\dots,K\}}\rho_k^{\tilde{\mathcal{L}}} -\max_{k\in\{1,2,\dots,K\}\setminus j_1}\rho_k^{\tilde{\mathcal{L}}} -\max_{k\in\{1,2,\dots,K\}\setminus \{j_1,j_2\}}\rho_k^l\right.\nonumber\\
&&\:\:\:\:\:\:\:\:\:\:\:\:\:\:\left.-\dots-\max_{k\in\{1,2,\dots,K\}\setminus \{j_1,j_2,\dots,j_{r-1}\}}\rho_k^{\tilde{\mathcal{L}}}\right)\nonumber\\
&=&\!\!\!\!\! \ln(2)r\epsilon_r(\mathcal{T}) \geq \epsilon_r(\mathcal{T}),
\label{eqn:entropyerr}
\end{eqnarray}
where the last inequality comes from the fact that $r \geq 1/\ln(2)$. Now recall that $G_1 \leq \ln K$ and normalizing $\kappa$ in Theorem ~\ref{thm:entropy} finishes the proof.
\end{proof}

\begin{proof}[Proof of Theorem~\ref{thm:errorwobalance}] The proof follows the same steps as Theorem~\ref{thm:entropy} until Equation~\ref{eqn:52}. Applying WHA~\ref{as:WHA4} at this point will result in the same result as in Equation~\ref{eqn:55}. The rest of the proof would be the same as Theorems~\ref{thm:entropy} and~\ref{thm:errorbound}.
\end{proof}

\begin{proof}[Proof of Theorem~\ref{thm:error1}]
Since we assume the objective is minimized in every node of the tree, therefore each node is sending examples to only one of its children and consequently each example descends to only one leaf. Thus in any leaf $l$, we store label histograms and assign first $r$ labels from the histogram to any example reaching that leaf, i.e. $y(x)=\{j_1,j_2,...,j_r\}$, where $j_1=\text{argmax}_{k \in \{1,2,\dots,K\}} \rho^l_k$, $j_2=\text{argmax}_{k \in \{1,2,\dots,K\}\setminus j_1} (\rho^l_k)$,..., $j_r=\text{argmax}_{k \in \{1,2,\dots,K\}\setminus \{j_1,...,j_{r-1}\}} (\rho^l_k)$ and $\rho_i^l$ is the probability that the data point $x$ has label $i$ given that $x$ has reached leaf $l$, i.e. $\rho_i^l = P(i \in t(x)|x \;\text{reached}\; l)$.

We next expand the $r$-level multi-label error as follows:
\begin{eqnarray}
\epsilon_r(\mathcal{T}) \!\!\!\!\!&=&\!\!\!\!\! \frac{1}{r}\sum_{i=1}^KP(i \in y_r(x),i \notin t(x))\\
\!\!\!\!\!&=&\!\!\!\!\! \frac{1}{r}\sum_{i=1}^KP(i \in t(x),i \notin y_r(x))\\
&=&\!\!\!\!\! \frac{1}{r}\sum_{l \in \mathcal{L}}\!w(l)\!\sum_{i=1}^K\!P(i \!\in\! t(x),i \!\notin\! y_r(x)|x \;\text{reached}\; l) \\
&=&\!\!\!\!\! \frac{1}{r}\sum_{l \in \mathcal{L}}w(l)\sum_{\substack{i = 1\\i\neq j_1,...,j_r}}^K\!\!\!P(i \in t(x)|x \;\text{reached}\; l) \\
&=&\!\!\!\!\! \frac{1}{r}\sum_{l \in \mathcal{L}}w(l)\left(\sum_{i=1}^K \rho_i^{(l)} - \max_{k\in\{1,2,\dots,K\}}\rho_k^l\right. \label{eqn:last} -\max_{k\in\{1,2,\dots,K\}\setminus j_1}\rho_k^l\\
&&\:\:\:\:\:\:\:\:\:\:\:\:\:\:\:\:\:\:\:\:\:\:\left. -\max_{k\in\{1,2,\dots,K\}\setminus \{j_1,j_2\}}\rho_k^l -\dots-\max_{k\in\{1,2,\dots,K\}\setminus \{j_1,j_2,\dots,j_{r-1}\}}\rho_k^l\right),\nonumber
\end{eqnarray}
where $w(l)$ denote the probability that example $x$ reaches leaf $l$ and $\mathcal{L}$ denote the set of all leaves of the tree. 
% \textcolor{red}{(removing cause it is repeated in previous the previous proof:
% Next we will find the Shannon entropy bound with respect to the error and show that the entropy of the tree, denoted as $G(\mathcal{T})$, upper-bounds the error. Note that:
% \begin{eqnarray}
% G(\mathcal{T}) &\coloneqq& \sum_{l \in \mathcal{L}} w(l) \sum_{i=1}^K \rho_i^l \ln\left(\frac{1}{\rho_i^l}\right)\\
% &\geq& \sum_{l \in \mathcal{L}} w(l) \sum_{\substack{i = 1\\i\neq j_1,...,j_R}}^K \rho_i^l \ln\left(\frac{1}{\rho_i^l}\right)
% \end{eqnarray}
% Note that $\sum_{i = 1}^K\rho_i^l = R$. Thus for any $i = 1,2,\dots,K$ such that $i\neq j_1,...,j_R$ it must hold that $\rho_i^l \leq \frac{1}{2}$. We continue as follows

% \begin{eqnarray}
% G(\mathcal{T}) \!\!\!\!\!&\geq&\!\!\!\!\! \sum_{l \in \mathcal{L}} w(l) \sum_{\substack{i = 1\\i\neq j_1,...,j_r}}^K \rho_i^l \ln(2)\\
% &\geq&\!\!\!\!\! \ln(2)\sum_{l \in \mathcal{L}}w(l)\left(\sum_{i=1}^K \rho_i^{(l)} - \max_{k\in\{1,2,\dots,K\}}\rho_k^l -\max_{k\in\{1,2,\dots,K\}\setminus j_1}\rho_k^l -\max_{k\in\{1,2,\dots,K\}\setminus \{j_1,j_2\}}\rho_k^l\right.\nonumber\\
% &&\:\:\:\:\:\:\:\:\:\:\:\:\:\:\left.-\dots-\max_{k\in\{1,2,\dots,K\}\setminus \{j_1,j_2,\dots,j_{R-1}\}}\rho_k^l\right)\nonumber\\
% &=&\!\!\!\!\! \ln(2)R\epsilon_R(\mathcal{T})
% \label{eqn:entropyerr}
% \end{eqnarray}
% }

From Lemma~\ref{lem:binary} (for binary tree) and Lemma~\ref{lem:mary} (for M-ary tree) it follows that for any node in the tree, the corresponding split is balanced and the following holds: $|P_j^i - P_{j^{'}}^i| = 1$ for all labels $i = 1,2,\dots,K$ and all pairs of children nodes $(j,j^{'})$ of the considered node such that $j,j^{'} \in \{1,2,\dots,M\}$ and $j \neq j^{'}$. Thus when splitting any node, its label histogram is divided in such a way that its children have non-overlapping label histograms, i.e. $\forall_{i = 1,2,\dots,K}\forall_{j,j^{'} \in \{1,2,\dots,M\}, j \neq j^{'}}\rho_i^{(j)}\rho_i^{(j^{'})} = 0$, where $\rho_i^{(j)}$ and $\rho_i^{(j^{'})}$ denote the $i^{\text{th}}$ entry in the normalized label histograms of children nodes $j$ and $j^{'}$ respectively. After $\log_M (K/r)$ splits we obtain leaves with non-overlapping histograms, i.e. for any two leaves $l_1$ and $l_2$ such that $l_1,l_2 \in \mathcal{L}$ and $l_1\neq l_2$, $\forall_{i = 1,2,\dots,K}\rho_i^{(l_1)}\cdot\rho_i^{(l_2)} = 0$. In each leaf the label histogram contains $r$ non-zero entries. Based on the above it follows that $G(\mathcal{T}) = 0$. Consequently, using Equation~\ref{eqn:entropyerr} we obtain that the multi-label error $\epsilon_r(\mathcal{T})$ is equal to zero as well. This directly implies that $\epsilon_{\hat{r}}(\mathcal{T}) = 0$ for any $\hat{r} = 1,2,\dots,r$.
\end{proof}

\begin{proof}[Proof of Lemma~\ref{lem:error3} (Proof of Lemma ~\ref{lem:error2} follows directly as Lemma ~\ref{lem:error2} is a special case of Lemma ~\ref{lem:error3})]
In our algorithm we store label histograms for each node, and at testing we assign to an example top $r$ labels obtained from averaging the histograms of the leaves to which this example has descended to. At training, we recursively find the node with the highest priority and partition it to two children. Here we are examining the change of error with one node split. We consider examples reaching that node and without loss of generality we assume they have reached only this node. For each such example $x$ we assign the top $r$ labels from the histogram of the analyzed node, i.e. $y_r(x)=\{k_1,k_2,...,k_r\}$, where $k_1=\text{argmax}_{k \in \{1,2,\dots,K\}} \rho_k$, $k_2=\text{argmax}_{k \in \{1,2,\dots,K\}\setminus j_1} (\rho_k)$,..., $k_r=\text{argmax}_{k \in \{1,2,\dots,K\}\setminus \{j_1,...,j_{r-1}\}} (\rho_k)$ and $\rho_i$ is the probability that the data point $x$ has label $i$ given that $x$ has reached node $n$, i.e. $\rho_i = P(i \in t(x)|x \;\text{reached}\; n)$. After $t$ splits the Precision can be expanded as follows:
\begin{eqnarray}
(P@r)^t \!\!\!\!\!&=&\!\!\!\!\! 
\frac{1}{r}\sum_{i=1}^KP(i \in t(x),i \in y_r(x))\\
&=&\!\!\!\!\! \frac{1}{r} \Big(\max_{k\in\{1,2,\dots,K\}}\rho_k + \max_{k\in\{1,2,\dots,K\}\setminus j_1}\rho_k +\dots+\max_{k\in\{1,2,\dots,K\}\setminus \{j_1,j_2,\dots,j_{r-1}\}}\rho_k \Big)\\
\!\!\!\!\!&=&\!\!\!\!\! 
\max_{k\in\{1,2,\dots,K\}}\pi_k  + \max_{k\in\{1,2,\dots,K\}\setminus j_1}\pi_k +\dots+\max_{k\in\{1,2,\dots,K\}\setminus \{j_1,j_2,\dots,j_{r-1}\}}\pi_k\\
\!\!\!\!\!&=&\!\!\!\!\!
\pi_{k_1} + \cdots + \pi_{k_r}, 
\end{eqnarray}
where the last line comes from the fact that $\pi_i$ is a normalized fraction of examples containing label $i$ in their labels. After the node split, the Precision is defined as the combination of the Precision of its children.
For simplicity we consider equal contribution of each of the edges to $P_{\text{multi}} = \left|\left(\sum_{j=1}^M P_j\right) - 1\right|$. Therefore we can write the Precisions of the children as:
\begin{eqnarray}
(P@r)^{t+1}
\!\!\!\!\!&=&\!\!\!\!\! (P_1-\frac{1}{M}P_{\text{multi}})(P@r)^1 + \cdots + (P_M-\frac{1}{M}P_{\text{multi}})(P@r)^M\\ \!\!\!\!\!&=&\!\!\!\!\! 
(P_1-\frac{1}{M}P_{\text{multi}})\Big(\max_{i\in\{1,2,\dots,K\}}\pi_i\big(\frac{P^i_1-\frac{1}{M}P^i_{\text{multi}}}{P_1-\frac{1}{M}P_{\text{multi}}}\big) + \cdots\Big) + \cdots\\
\!\!\!\!\!&+&\!\!\!\!\! 
(P_M-\frac{1}{M}P_{\text{multi}})\Big(\max_{j\in\{1,2,\dots,K\}}\pi_j\big(\frac{P^j_M-\frac{1}{M}P^j_{\text{multi}}}{P_M-\frac{1}{M}P_{\text{multi}}}\big) + \cdots\Big)\nonumber\\
\!\!\!\!\!&=&\!\!\!\!\! 
\max_{i\in\{1,2,\dots,K\}}\pi_i (P^i_1-\frac{1}{M}P^i_{\text{multi}}) + \cdots \\
\!\!\!\!\!&+&\!\!\!\!\! 
\max_{j\in\{1,2,\dots,K\}}\pi_j (P^j_M-\frac{1}{M}P^j_{\text{multi}}) + \cdots \nonumber\\
\!\!\!\!\!&=&\!\!\!\!\! 
\frac{1}{M}\big(\max_{i\in\{1,2,\dots,K\}}\pi_i ((M-1)P^i_1-P^i_2 \cdots -P^i_M+1) + \cdots\\
\!\!\!\!\!&+&\!\!\!\!\! 
\max_{j\in\{1,2,\dots,K\}}\pi_j ((M-1)P^i_M-P^i_1 \cdots -P^i_{M-1}+1) + \cdots\big) \nonumber\\
\!\!\!\!\!&=&\!\!\!\!\! 
\frac{1}{M}\big(\max_{i\in\{1,2,\dots,K\}}\pi_i ((P^i_1-P^i_2) + (P^i_1-P^i_3) + \cdots  (P^i_1-P^i_M)+1) + \cdots\\
\!\!\!\!\!&+&\!\!\!\!\! 
\max_{j\in\{1,2,\dots,K\}}\pi_j ((P^i_M-P^i_1) + (P^i_M-P^i_2)+ \cdots (P^i_M-P^i_{M-1})+1) + \cdots\big).\nonumber
\end{eqnarray}
Note that the subtraction of $(1/M)P^i_{\text{multi}}$ and $(1/M)P_{\text{multi}}$ in the coefficients is done to compensate the Precision calculation for examples being sent to multiple directions. Let the top $r$ labels assigned to the first child be denoted as $y_r^1(x)=\{i_1,i_2,...,i_r\}$, where\\ $i_1=\text{argmax}_{i \in \{1,2,\dots,K\}} \pi_i ((P^i_1-P^i_2) + (P^i_1-P^i_3) + \cdots  (P^i_1-P^i_M))$,\\ $i_2=\text{argmax}_{k \in \{1,2,\dots,K\}\setminus i_1} \pi_i ((P^i_1-P^i_2) + (P^i_1-P^i_3) + \cdots  (P^i_1-P^i_M))$,\\...,\\ $i_r=\text{argmax}_{k \in \{1,2,\dots,K\}\setminus \{i_1,...,i_{r-1}\}} \pi_i ((P^i_1-P^i_2) + (P^i_1-P^i_3) + \cdots  (P^i_1-P^i_M))$.\\
Analogy holds for all other children. Thus for example the $M^{\text{th}}$ children's labels are: $y_r^M(x)=\{j_1,j_2,...,j_r\}$. Therefore the difference between the Precision of the parent node and its children can be written as:
\begin{eqnarray}
(P@r)^{t+1} - (P@r)^{t}
\!\!\!\!\!&=&\!\!\!\!\! 
\frac{1}{M}\Big(\pi_{i_1} ((P^{i_1}_1-P^{i_1}_2)+ \cdots (P^{i_1}_1-P^{i_1}_M)+1) + \cdots\\
\!\!\!\!\!&+&\!\!\!\!\! 
\pi_{i_r} ((P^{i_r}_1-P^{i_r}_2)+ \cdots (P^{i_r}_1-P^{i_r}_M)+1)\Big)
\nonumber\\
\!\!\!\!\!&+&\!\!\!\!\! 
\cdots \nonumber\\
\!\!\!\!\!&+&\!\!\!\!\! 
\frac{1}{M}\Big(\pi_{j_1} ((P^{j_1}_M-P^{j_1}_1)+ \cdots (P^{j_1}_M-P^{j_1}_{M-1})+1) + \cdots\nonumber\\
\!\!\!\!\!&+&\!\!\!\!\! 
\pi_{j_r} ((P^{j_r}_M -P^{j_r}_1)+ \cdots (P^{j_r}_M-P^{j_r}_{M-1})+1)\Big)\nonumber\\
\!\!\!\!\!&-&\!\!\!\!\! 
\big (\pi_{k_1} + \cdots + \pi_{k_r} \big ).\nonumber
\end{eqnarray}
For the ease of notation we show the case for the binary below:
\begin{eqnarray}
(P@r)^{t+1} - (P@r)^{t}
\!\!\!\!\!&=&\!\!\!\!\!
% \frac{1}{2}\big(\max_{i\in\{1,2,\dots,K\}}\pi_i (P^i_R-P^i_L+1) + \cdots\\
% \!\!\!\!\!&+&\!\!\!\!\! 
% \max_{j\in\{1,2,\dots,K\}}\pi_j (P^j_L-P^j_R+1) + \cdots\big) \nonumber\\
% \!\!\!\!\!&-&\!\!\!\!\! 
% \big (\max_{k\in\{1,2,\dots,K\}}\pi_k + \cdots \big ) \nonumber\\
% \!\!\!\!\!&=&\!\!\!\!\!
\frac{1}{2} \big(\pi_{i_1}(P^{i_1}_R-P^{i_1}_L+1) + \cdots + \pi_{i_r}(P^{i_r}_R-P^{i_r}_L+1) \big)\\
\!\!\!\!\!&+&\!\!\!\!\! 
\frac{1}{2} \big(\pi_{j_1}(P^{j_1}_L-P^{j_1}_R+1) + \cdots + \pi_{j_r}(P^{j_r}_L-P^{j_r}_R+1) \big) \nonumber \\
\!\!\!\!\!&-&\!\!\!\!\! 
\big (\pi_{k_1} + \cdots + \pi_{k_r} \big ).\nonumber
\end{eqnarray}
Considering the Assumption~\ref{as:WHA1},we have at least one label such that $P_R^{k}-P_L^{k} = \gamma_1 > 0, \gamma_1 \in (0,1]$. 
Without loss of generality let $P_{R}^{k_1}-P_L^{k_1} = \gamma_1 > 0$ for the top label in the parent node. 
Thus: $\pi_{i_1}(P^{i_1}_R-P^{i_1}_L+1) \geq \pi_{k_1} (1+\gamma_1)$ and $\pi_{j_1}(P^{j_1}_L-P^{j_1}_R+1) \geq \pi_{k_1} (1-\gamma_1)$. Therefore we have $(P@r)^{t+1} - (P@r)^{t} \geq 0$. Due to the weak hypothesis assumption the histograms in the children nodes are different than in the parent on at least one position corresponding to one label. If that label is in the top $r$ labels that we assign to the children node, the error will be reduced. If not, the error is going to be the same, but that cannot happen forever, i.e. for some split the label(s) for which the weak hypothesis assumption holds will eventually be in the top $r$ labels that are assigned to the children node. To put this intuition into more formal language, if any of the top $r$ labels in any of the children are different from the top $r$ parent labels, i.e. $y_r^1 \neq y_r$, $y_r^2 \neq y_r$,...,  or $y_r^M \neq y_r$ we will have $(P@r)^{t+1} - (P@r)^{t} > 0$. Because of the weak hypothesis assumption, the latter condition is inevitable and will eventually hold after some node split. This shows that the error is monotonically decreasing.
\end{proof}

\section{ADDITIONAL ALGORITHMS}
\vspace{-0.7in}
\begin{minipage}[t]{0.48\textwidth}
\vspace{0.4in}
 \setcounter{algorithm}{2}
\begin{algorithm}[H]
\caption{\textbf{OptimizeObjective ($v$)}}
\begin{algorithmic} 
\STATE $J_{opt} \leftarrow +\infty$
\FOR{$s=1\dots 2^M-1$}
\FOR{$m=1\dots M$}
\STATE $\hat{y}[m] = s \wedge 2^{(m-1)} > 0$
\STATE $P_m \leftarrow \frac{(v.C_v - y_i.size())v.P_m + y_i.size()*\hat{y}[m]}{v.C_v}$
\FOR{$k \in y_i$}
\STATE $P_m^k \leftarrow \frac{(v.l_v[k] - 1)v.P_m^k + \hat{y}[m]}{v.l_v[k]}$
\ENDFOR
\ENDFOR
\vspace{0.05in}
\STATE \% objective computation
\STATE $B \leftarrow \sum_{j = 1}^M\sum_{l=j+1}^M\left|P_j-P_l\right|$
\STATE $CI \leftarrow \sum_{i=1}^{y_i.size()} \sum_{j=1}^M\sum_{l=j+1}^M\frac{v.l_v(i)}{v.C_v}\left|P_j^i-P_l^i\right|$
\STATE $MWP \leftarrow \left|\left(\sum_{j=1}^MP_j\right)-1\right|$
\STATE $J \leftarrow B - \lambda_1 CI + \lambda_2 MWP$
\vspace{0.05in}
\IF{$J < J_{opt}$}
\STATE $J_{opt} \leftarrow J$
\STATE $\hat{y}_{opt} \leftarrow \hat{y}$
\ENDIF
\ENDFOR
\STATE {\bf return} $\hat{y}_{opt}$
\end{algorithmic}
\label{alg:OO}
\end{algorithm}
\setlength{\textfloatsep}{0.3cm}
\setlength{\floatsep}{0.3cm}
\end{minipage} 
\begin{minipage}[t]{0.48\textwidth}
\vspace{0.2in}
\begin{algorithm}[H]
\caption{\textbf{TrainRegressors ($v$)}}
\begin{algorithmic} 
\STATE \% $y_i.size()$ denotes the size of vector $y_i$
\STATE $v.C_v  \leftarrow 0$;\:\:\:\:\:$v.l_v \leftarrow \oldemptyset$;\:\:\:\:\:$v.isLeaf \leftarrow false$
\FOR{$m=1\dots M$}
\STATE $v.w_m \leftarrow$ random weights;\:\:\:\:\:$v.P_m \leftarrow 0$\\
\textbf{for} i=1\dots K \textbf{do}\;\; $v.P_m^i \leftarrow 0$\;\; \textbf{end for}
% \FOR{i=1\dots K} 
% \STATE $v.P_m^i \leftarrow 0$ 
% \ENDFOR
\ENDFOR
\FOR{$e=1\dots E$}
\FOR{$i \in v.I$} 
\FOR{$k \in y_i$}
\STATE $v.C_v\texttt{++}$;\:\:\:\:\:$v.l_{v}[k]\texttt{++}$
\ENDFOR 
\STATE $\hat{y} \leftarrow$ {\bf OptimizeObjective ($v$)}
\FOR{$m=1\dots M$}
\STATE {\bf Train} $v.w_m$ with example $(x_i, \hat{y}[m])$
\STATE $pred \leftarrow clamp_{[0,1]}(v.w_m^T x_i)$
\STATE $v.P_m \leftarrow$ \\ 
$\frac{(v.C_v - y_i.size()))*v.P_m + y_i.size()*pred}{v.C_v}$
\FOR{$k \in y_i$}
\STATE $v.P_m^k \leftarrow \frac{(v.l_v[k] - 1)*v.P_m^k + pred}{v.l_v[k]}$
\ENDFOR
\ENDFOR
\ENDFOR
\ENDFOR
\end{algorithmic}
\label{alg:WU}
\end{algorithm}
\end{minipage}
\begin{minipage}[t]{1\textwidth}
\vspace{-0.1in}
\begin{algorithm}[H]
\caption{\textbf{CreateChildren ($v$)}}
\begin{algorithmic} 
\FOR{$m=1\dots M$}
\STATE $v.ch[m].I \leftarrow \oldemptyset$
\STATE $v.ch[m].Lhist \leftarrow \oldemptyset$
\STATE $v.ch[m].isLeaf \leftarrow true$ 
\ENDFOR
\FOR{$i \in v.I$}
\STATE $sent \leftarrow false$
\FOR{$m \in 1\dots M$}
\IF{$v.w_m^{\top}x_i>0.5$}
\STATE \% example $(x_i, y_i)$ goes to child $m$
\STATE {\bf UpdateHist ($v.ch[m].Lhist$, $y_i$)} 
\STATE $v.ch[m].I.push(i)$
\STATE $sent \leftarrow true$
\ENDIF
\ENDFOR
\IF{{\bf not} $sent$}
\STATE $m \leftarrow \arg\max_{\hat{m}\in \{1,2,\dots,M\}}{v.w_{\hat{m}}^{\top}x_i}$
\STATE {\bf UpdateHist ($v.ch[m].Lhist$, $y_i$)} 
\STATE $v.ch[m].I.push(i)$
\ENDIF
\ENDFOR
\STATE {\bf return} $v.ch$ 
\end{algorithmic}
\label{alg:CC}
\end{algorithm}
\end{minipage}

\clearpage
\newpage
\section{EXPERIMENTAL SETUP}

LdSM was implemented in C++. The regressors in the tree nodes were trained with either SGD [\cite{bottou-98x}] (Mediamill) or NAG [\cite{DBLP:journals/corr/abs-1305-6646}] (remaining data sets) with step size chosen from $[0.001,1]$. The trees were trained with up to $20$ passes through the data and we explored trees with up to $64K$ nodes for  Mediamill and Bibtex, up to $32K$ for Delicious, and up to $2K$ for the rest of the data sets. $\lambda_1$ and $\lambda_2$ were chosen from the set $\{0.5,1,1.5, 2, 4\}$ and $M$ was set to either $2$ or $4$. FastXML, PFastreXML, CRAFTML and LdSM algorithms use tree ensembles of size $\sim 50$. PLT and LPSR use a single tree, and GBDT-S uses up to $100$ trees. 
\begin{table}[htp!]
\centering
 \setlength{\tabcolsep}{3pt}
 \caption{Data set statistics.}
 \begin{tabular}{|c||c|c|c|c|c|c|}\hline

\multirow{2}{*}{Data Sets} & \multirow{2}{*}{\#Features} & \multirow{2}{*}{\#Labels} & \#Training & \#Testing & Avg. Labels & Avg. Points \\ 
& & & samples& samples&  per Point &  per Label\\ \hline
\hline
Mediamill & 120&	101&	30993&	12914 &4.38 &1902.15	\\
 \hline
 Bibtex & 1836&	159&	4880&	2515&		2.40&111.71	\\
 \hline
Delicious &500&	983 &	12920&	3185& 19.03&	311.61\\
\hline
Eurlex & 5000&	3993&	15539&	3809& 5.31&	25.73	\\
 \hline
AmazonCat-13k & 203882&	13330&	1186239&	306782	& 5.04&	448.57 \\
\hline
Wiki10-31k &
101938&	30938	&14146&	6616&	18.64 &8.52\\
\hline
Delicious-200k &782585&	205443	&196606	&100095&	75.54 &72.29\\
\hline
% WikiLSHTC-325k &1617899	&325056&	1778351&	587084&	3.19&17.46\\
% \hline
Amazon-670k&135909	&670091&	490449&	153025	&5.45&3.99\\
\hline
\end{tabular}
\end{table}

\begin{table}[htp!]
\caption{Experimental setup that was used to obtain results for various data sets with LdSM method: the depth of the deepest tree in the ensemble and tree arity.}
    \centering
    \begin{tabular}{|c||c|c|}
    \hline
    Data sets& Depth& Arity\\
    \hline
         Mediamill&9 &4  \\
         \hline
         Bibtex& 9&4\\
         \hline
         Delicious&10&4\\
         \hline
         AmazonCat-13k&18&2\\
         \hline
         Wiki10-31k&10&4\\
         \hline
         Delicious-200k&46&2\\
         \hline
        %  WikiLSHTC-325k&22&2\\
        %  \hline
         Amazon-670k&25&2\\
         \hline
    \end{tabular}
    \label{tab:dM}
\end{table}

\clearpage
\newpage
\section{ADDITIONAL EXPERIMENTAL RESULTS}
% \section{Additional Experimental Results}
\begin{table}[htp!]
% \scriptsize
%\setlength{\tabcolsep}{0.3pt}
\caption{Prediction time [ms] per example for tree-based approaches: GBDT-S, CRAFTML, FastXML, PFastreXML, LdSM (LPSR and PLT are NA) and other (not purely tree-based) methods: Parabel, DisMEC~\cite{Babbar:2017:DDS:3018661.3018741}, PD-Sparse~\cite{DBLP:conf/icml/YenHRZD16}, PPD-Sparse~\cite{DBLP:conf/kdd/YenHDRDX17}, OVA-Primal++ ~\cite{SDM2019} and SLEEC~\cite{NIPS2015_5969}  on various data sets. The best result among tree-based methods is in bold, and among all methods is underlined.}
\setlength{\tabcolsep}{3pt}
\vspace{0.05in}
\centering
% \begin{tabular}{c|| c|c|c|c|c||c|c|c|c|| c|}
% \cline{2-10} 
% &\multicolumn{5}{c||}{Other}& \multicolumn{4}{c||}{Tree-based}\\
% \cline{2-11} 
% &\multirow{1}{*}{} Parabel& DiSMEC&PD-Sparse&PPD-Sparse&SLEEC& GBDT-S&CRAFTML&FastXML & PFastreXML & LdSM \\
% \hline
% \hline
% \multicolumn{1}{|c||}{\multirow{1}{*}{Mediamill}}&NA&0.142&0.004&0.078&4.95 &0.05&NA& 0.27 & 0.37 & \textbf{0.05} \\
% \hline
% \multicolumn{1}{|c||}{\multirow{1}{*}{Bibtex}} &NA&0.28&0.007&0.094&0.70&NA&NA& 0.64 & 0.73 & \textbf{0.04}\\
% \hline
% \multicolumn{1}{|c||}{\multirow{1}{*}{Delicious}} &NA&NA&NA&NA&NA&0.04&NA& NA & NA & \textbf{-}\\
% \hline
% \multicolumn{1}{|c||}{\multirow{1}{*}{Eurlex-4k}} &1.01~\cite{Prabhu18}&0.73(7.05~\cite{Prabhu18})&1.5(0.12 ~\cite{Prabhu18})&2.26(1.14~\cite{Prabhu18})&3.67(7.57~\cite{Prabhu18})&NA&5.39& 3.65 & 5.43(3.92~\cite{Prabhu18}) & \textbf{0.08}\\
% \hline
% \multicolumn{1}{|c||}{\multirow{1}{*}{AmazonCat-13k}} &NA&0.20&0.87&1.82&13.36&NA&5.12& 1.21 & 1.34 & \textbf{0.49}\\
% \hline
% \multicolumn{1}{|c||}{\multirow{1}{*}{Wiki10-31k}} &NA&NA&NA&NA&NA&\textbf{0.20}&NA& 3.00(?) & NA & -\\
% \hline
% \multicolumn{1}{|c||}{\multirow{1}{*}{Delicious-200k}} &NA&311.4&0.43&275&2.69&\textbf{0.14}&8.6& 1.28 & 7.40 & -\\
% \hline
% \multicolumn{1}{|c||}{\multirow{1}{*}{Amazon-670k}} &1.13&148(1414~\cite{Prabhu18})&NA&20(66.09~\cite{Prabhu18})&6.94(18.51~\cite{Prabhu18})&NA&5.02& \textbf{1.48} & 1.98(4.75~\cite{Prabhu18}) & -\\
% \hline
% \end{tabular}
\setlength{\tabcolsep}{3pt}
\begin{tabular}{c||c|c|c|c|| c|}
\cline{2-6} 
&\multicolumn{5}{|c|}{Tree-based}\\
\cline{2-6} 
&\multirow{1}{*}{} GBDT-S&CRAFTML&FastXML & PFastreXML & LdSM \\
\hline
\hline
\multicolumn{1}{|c||}{\multirow{1}{*}{Mediamill}}&0.05&NA& 0.27 & 0.37 & \textbf{0.05} \\
\hline
\multicolumn{1}{|c||}{\multirow{1}{*}{Bibtex}}& NA&NA& 0.64 & 0.73 & \textbf{0.013}\\
\hline
\multicolumn{1}{|c||}{\multirow{1}{*}{Delicious}} &0.04&NA& NA & NA & \underline{\textbf{0.014}}\\
\hline
% \multicolumn{1}{|c||}{\multirow{1}{*}{Eurlex-4k}} &NA&5.39& 3.65 & 5.43 & \textbf{\underline{0.03}}\\
% \hline
\multicolumn{1}{|c||}{\multirow{1}{*}{AmazonCat-13k}} &NA&5.12& 1.21 & 1.34 & \underline{\textbf{0.04}}\\
\hline
\multicolumn{1}{|c||}{\multirow{1}{*}{Wiki10-31k}} &0.20&NA& 1.38 & NA & \underline{\underline{\textbf{0.15}}}\\
\hline
\multicolumn{1}{|c||}{\multirow{1}{*}{Delicious-200k}} &\textbf{0.14}&8.6& 1.28 & 7.40 & 1.21\\
\hline
% \multicolumn{1}{|c||}{\multirow{1}{*}{WikiLSHTC-325k}} &NA&7.67& \textbf{1.02} & 1.47 & 2.77\\
% \hline
\multicolumn{1}{|c||}{\multirow{1}{*}{Amazon-670k}} &NA&5.02& 1.48 & 1.98 & \underline{\textbf{0.12}}\\
\hline
\end{tabular}
% \hspace{0.2in}
\begin{tabular}{c|| c|c|c|c|c|c|}
\cline{2-7} 
&\multicolumn{6}{c|}{Other}\\
\cline{2-7} 
&\multirow{1}{*}{} Parabel& DiSMEC&PD-Sparse&PPD-Sparse&OVA-Primal++&SLEEC \\
\hline
\hline
\multicolumn{1}{|c||}{\multirow{1}{*}{Mediamill}}&NA&0.142&\underline{ 0.004}&0.078&NA&4.95 \\
\hline
\multicolumn{1}{|c||}{\multirow{1}{*}{Bibtex}} &NA&0.28&\underline{0.007}&0.094&NA&0.70\\
\hline
\multicolumn{1}{|c||}{\multirow{1}{*}{Delicious}} &NA&NA&NA&NA&NA&NA\\
\hline
% \multicolumn{1}{|c||}{\multirow{1}{*}{Eurlex-4k}} &1.01&0.73&1.5&2.26&NA&3.67\\
% \hline
\multicolumn{1}{|c||}{\multirow{1}{*}{AmazonCat-13k}} &NA&0.20&0.87&1.82&NA&13.36\\
\hline
\multicolumn{1}{|c||}{\multirow{1}{*}{Wiki10-31k}} &NA&116.66&NA&NA&NA&NA\\
\hline
\multicolumn{1}{|c||}{\multirow{1}{*}{Delicious-200k}} &NA&311.4&\underline{0.43}&275&NA&2.69\\
\hline
% \multicolumn{1}{|c||}{\multirow{1}{*}{WikiLSHTC-325k}} &1.17&65&3.89& 290&NA&4.85\\
% \hline
\multicolumn{1}{|c||}{\multirow{1}{*}{Amazon-670k}} &1.13&148&NA&20&NA&6.94\\
\hline
\end{tabular}
\label{tab:predtimeall}
% \vspace{-1in}
\end{table}

\begin{table}[H]
% \scriptsize
%\setlength{\tabcolsep}{0.3pt}
\caption{Training time [s] for tree-based approaches: GBDT-S, CRAFTML, FastXML, PFastreXML, LdSM (LPSR and PLT are NA) and other (not purely tree-based) methods: Parabel, DisMEC, PD-Sparse, PPD-Sparse, SLEEC,  on various data sets. The best result among tree-based methods is in bold, and among all methods is underlined.}
\setlength{\tabcolsep}{3pt}
\vspace{0.05in}
\centering
\begin{tabular}{c||c|c|c|c|| c|}
\cline{2-6} 
&\multicolumn{5}{|c|}{Tree-based}\\
\cline{2-6} 
&\multirow{1}{*}{} GBDT-S&CRAFTML&FastXML & PFastreXML & LdSM \\
\hline
\hline
\multicolumn{1}{|c||}{\multirow{1}{*}{Mediamill}} &NA&NA& 276.4 & 293.2 & \textbf{52.7} \\
\hline
\multicolumn{1}{|c||}{\multirow{1}{*}{Bibtex}} &NA&NA& 21.68 & 21.47& \textbf{9.48}\\
\hline
\multicolumn{1}{|c||}{\multirow{1}{*}{Delicious}} &NA&NA& NA & NA & \underline{\textbf{21.74}}\\
\hline
% \multicolumn{1}{|c||}{\multirow{1}{*}{Eurlex-4k}}&NA&\textbf{47.75}& 315.9&324.4 & 82.70\\
% \hline
\multicolumn{1}{|c||}{\multirow{1}{*}{AmazonCat-13k}} &NA&2876& 11535&13985 & \textbf{607}\\
\hline
\multicolumn{1}{|c||}{\multirow{1}{*}{Wiki10-31k}}&1044&NA& 1275.9 & NA & \underline{\textbf{179}}\\
\hline
\multicolumn{1}{|c||}{\multirow{1}{*}{Delicious-200k}}&NA&\underline{\textbf{1174}}& 8832.46 & 8807.51 & 5125\\
\hline
% \multicolumn{1}{|c||}{\multirow{1}{*}{WikiLSHTC-325k}}& NA&\textbf{5092}&19160 &20070&124131\\
% \hline
\multicolumn{1}{|c||}{\multirow{1}{*}{Amazon-670k}}&NA&1487& 5624& 6559& \textbf{957}\\
\hline
\end{tabular}
\begin{tabular}{c|| c|c|c|c|c|c|}
\cline{2-7} 
&\multicolumn{6}{c|}{Other}\\
\cline{2-7} 
&\multirow{1}{*}{} Parabel& DiSMEC&PD-Sparse&PPD-Sparse&OVA-Primal++&SLEEC \\
\hline
\hline
\multicolumn{1}{|c||}{\multirow{1}{*}{Mediamill}} &NA&\underline{12.15}&34.1&23.8&NA&9504 \\
\hline
\multicolumn{1}{|c||}{\multirow{1}{*}{Bibtex}} &NA&\underline{0.203}&7.71&0.232&NA&296.86\\
\hline
% \multicolumn{1}{|c||}{\multirow{1}{*}{Eurlex-4k}} &226.8&76.07&773.2&\underline{9.95}&33.90&4543\\
% \hline
\multicolumn{1}{|c||}{\multirow{1}{*}{Delicious}} &NA&NA&NA&NA&NA&NA\\
\hline
\multicolumn{1}{|c||}{\multirow{1}{*}{AmazonCat-13k}}&NA&11828&2789&\underline{122.8}&7330&119840 \\
\hline
\multicolumn{1}{|c||}{\multirow{1}{*}{Wiki10-31k}}&NA&NA&NA&NA&1364&NA\\
\hline
\multicolumn{1}{|c||}{\multirow{1}{*}{Delicious-200k}}&NA&38814&5137.4&2869&NA&4838.7\\
\hline
% \multicolumn{1}{|c||}{\multirow{1}{*}{WikiLSHTC-325k}} &13032&271407&94343 & \underline{353}&NA&39000\\
% \hline
\multicolumn{1}{|c||}{\multirow{1}{*}{Amazon-670k}}&1512&174135&NA&\underline{921.9}&NA&20904\\
\hline
\end{tabular}
\label{tab:traintimesecall}
\vspace{-0.1in}
\end{table}

\begin{remark}[Training time]
The training time of LdSM can be reduced order of magnitudes by using lower number of epochs at the expense of $\sim 1 \%$ loss in the accuracy. However, we report the training times that correspond to the best accuracy results obtained with LdSM. 
\end{remark}

\begin{table}[htp!]
\caption{Propensity Score Precisions: $PSP@1$, $PSP@3$, and $PSP@5$ ($\%$) and Propensity Score nDCG scores: $PSN@1$, $PSN@3$, and $PSN@5$ ($\%$) obtained by different tree-based methods on common multi-label data sets.}
\label{tab:accup}
% \vspace{0.025in}
\setlength{\tabcolsep}{0.5pt}
\centering
\begin{tabular}{c|| c| c| c||c| c| c|}
\hline
\multicolumn{7}{|c|}{Mediamill $D = 120, K = 101$ }\\
\hline
\multicolumn{1}{|c||}{\multirow{1}{*}{Algorithm}} & PSP@1 & PSP@3 & PSP@5 & PSN@1 & PSN@3 & PSN@5\\
\hline
\multicolumn{1}{|c||}{\multirow{1}{*}{LPSR}} & 66.06& 63.83 & 61.11 & 66.06& 64.83 & 62.94\\
\hline
% \multicolumn{1}{|c||}{\multirow{1}{*}{PLT}} & -& - & - & -& -& -\\
% \hline
% \multicolumn{1}{|c||}{\multirow{1}{*}{GBDT-S}} & 84.23& 67.85 & - & -& -& -\\
% \hline
% \multicolumn{1}{|c||}{\multirow{1}{*}{CRAFTML}} & 85.86& 69.01 & 54.65 & -& -& -\\
% \hline
\multicolumn{1}{|c||}{\multirow{1}{*}{FastXML}} & 66.67& 65.43& 64.30& 66.67& 66.08& 65.24\\
\hline
\multicolumn{1}{|c||}{\multirow{1}{*}{PFastreXML}} &66.88& 65.90& 64.90& 66.88& 66.47& 65.71\\
\hline
\hline
\multicolumn{1}{|c||}{\multirow{1}{*}{LdSM}}& \textbf{70.27} & \textbf{69.66} & \textbf{68.86} &\textbf{70.27} &\textbf{69.99}&\textbf{70.30} \\
\hline
\end{tabular}
\vspace{0.1in}

\centering
\setlength{\tabcolsep}{0.5pt}
\begin{tabular}{c|| c| c| c||c| c| c|}
\hline
\multicolumn{7}{|c|}{Bibtex $D = 1.8k, K = 159$ }\\
\hline
\multicolumn{1}{|c||}{\multirow{1}{*}{Algorithm}} & PSP@1 & PSP@3 & PSP@5 & PSN@1 & PSN@3 & PSN@5\\
\hline
\multicolumn{1}{|c||}{\multirow{1}{*}{LPSR}} &  49.20& 50.14 & 55.01 &49.20 & 49.78& 52.41\\
\hline
% \multicolumn{1}{|c||}{\multirow{1}{*}{PLT}} &   -& - & - & -& -& -\\
% \hline
% \multicolumn{1}{|c||}{\multirow{1}{*}{GBDT-S}} &  -& - & - & -& -& -\\
% \hline
% \multicolumn{1}{|c||}{\multirow{1}{*}{CRAFTML}} & \textbf{65.15}& \textbf{39.83} & 28.99& -& - &-\\
% \hline
\multicolumn{1}{|c||}{\multirow{1}{*}{FastXML}} &  48.54& 52.30& 58.28& 48.54& 51.11& 54.38 \\
\hline
\multicolumn{1}{|c||}{\multirow{1}{*}{PFastreXML}} &  \textbf{52.28}& 54.36& \textbf{60.55}& \textbf{52.28}& 53.62&56.99\\
\hline
\hline
\multicolumn{1}{|c||}{\multirow{1}{*}{LdSM}} &  52.01& \textbf{54.38}& 60.34& 52.01& \textbf{53.67}& \textbf{57.08}  \\
\hline
\end{tabular}
\vspace{0.1in}

\setlength{\tabcolsep}{0.5pt}
\centering
\begin{tabular}{c|| c| c| c||c| c| c|}
\hline
\multicolumn{7}{|c|}{Delicious $D = 500, K = 983$ }\\
\hline
\multicolumn{1}{|c||}{\multirow{1}{*}{Algorithm}} & PSP@1 & PSP@3 & PSP@5 & PSN@1 & PSN@3 & PSN@5\\
\hline
\multicolumn{1}{|c||}{\multirow{1}{*}{LPSR}} & 31.34&  32.57& 32.77 &31.34 & 32.29& 32.50\\
\hline
% \multicolumn{1}{|c||}{\multirow{1}{*}{PLT}} & -& - & - & -& -& -\\
% \hline
% \multicolumn{1}{|c||}{\multirow{1}{*}{GBDT-S}} & 69.29& 63.62 & - & -& -& -\\
% \hline
% \multicolumn{1}{|c||}{\multirow{1}{*}{CRAFTML}} & 70.26& 63.98 & 59.00 & -& -& -\\
% \hline
\multicolumn{1}{|c||}{\multirow{1}{*}{FastXML}} &32.35& 34.51& 35.43& 32.35& 34.00& 34.73 \\
\hline
\multicolumn{1}{|c||}{\multirow{1}{*}{PFastreXML}} & 34.57& 34.80& 35.86& 34.57& 34.71& 35.42\\
\hline
\hline
\multicolumn{1}{|c||}{\multirow{1}{*}{LdSM}} & \textbf{37.27} & \textbf{38.32} & \textbf{38.46} &\textbf{37.27} & \textbf{38.09}&\textbf{38.28} \\
\hline
\end{tabular}
\vspace{0.1in}

\setlength{\tabcolsep}{0.5pt}
\centering
\begin{tabular}{c|| c| c| c||c| c| c|}
\hline
\multicolumn{7}{|c|}{AmazonCat-13k $D = 204k, K = 13k$}\\
\hline
\multicolumn{1}{|c||}{\multirow{1}{*}{Algorithm}} & PSP@1 & PSP@3 & PSP@5 & PSN@1 & PSN@3 & PSN@5\\
\hline
\multicolumn{1}{|c||}{\multirow{1}{*}{LPSR}} &-& - & - &- &- &- \\
\hline
% \multicolumn{1}{|c||}{\multirow{1}{*}{PLT}} &91.47& 75.84 & 61.02 & -& -& -\\
% \hline
% \multicolumn{1}{|c||}{\multirow{1}{*}{GBDT-S}} &-& - & - & -& -& -\\
% \hline
% \multicolumn{1}{|c||}{\multirow{1}{*}{CRAFTML}} &92.78& 78.48 & 63.58 & -& -& -\\
% \hline
\multicolumn{1}{|c||}{\multirow{1}{*}{FastXML}} & 48.31& 60.26& 69.30& 48.31& 56.90& 62.75 \\
\hline
\multicolumn{1}{|c||}{\multirow{1}{*}{PFastreXML}} &\textbf{69.52}& \textbf{73.22}& \textbf{75.48}& \textbf{69.52}& \textbf{72.21}& \textbf{73.67} \\
\hline
\hline
\multicolumn{1}{|c||}{\multirow{1}{*}{LdSM}} &51.06& 58.67& 60.47& 51.06& 57.78& 60.52 \\
\hline
\end{tabular}
\centering
\vspace{0.1in}

\setlength{\tabcolsep}{0.5pt}
\begin{tabular}{c|| c| c| c||c| c| c|}
\hline
\multicolumn{7}{|c|}{Wiki10-31k $D = 102k, K = 31k$ }\\
\hline
\multicolumn{1}{|c||}{\multirow{1}{*}{Algorithm}} & PSP@1 & PSP@3 & PSP@5 & PSN@1 & PSN@3 & PSN@5\\
\hline
\multicolumn{1}{|c||}{\multirow{1}{*}{LPSR}}&12.79&12.26  & 12.13&12.79 & 12.38& 12.27\\
\hline
% \multicolumn{1}{|c||}{\multirow{1}{*}{PLT}}&84.34& 72.34& 62.72& -& -& -\\
% \hline
% \multicolumn{1}{|c||}{\multirow{1}{*}{GBDT-S}}&84.34& 70.82 & - & -& -& -\\
% \hline
% \multicolumn{1}{|c||}{\multirow{1}{*}{CRAFTML}}&\textbf{85.19}& \textbf{73.17} & \textbf{63.27} & -& -& -\\
% \hline
\multicolumn{1}{|c||}{\multirow{1}{*}{FastXML}}&9.80& 10.17& 10.54& 9.80& 10.08& 10.33\\
\hline
\multicolumn{1}{|c||}{\multirow{1}{*}{PFastreXML}}&\textbf{19.02}&\textbf{18.34}& \textbf{18.43}& \textbf{19.02}& \textbf{18.49}& \textbf{18.52} \\
\hline
\hline
\multicolumn{1}{|c||}{\multirow{1}{*}{LdSM}}&11.87& 12.35&12.89& 11.87& 12.42& 12.58\\
\hline
\end{tabular}
\centering
\vspace{0.1in}

\setlength{\tabcolsep}{0.5pt}
\begin{tabular}{c|| c| c| c||c| c| c|}
\hline
\multicolumn{7}{|c|}{Delicious-200k $D = 783k, K = 205k$}\\
\hline
\multicolumn{1}{|c||}{\multirow{1}{*}{Algorithm}} & PSP@1 & PSP@3 & PSP@5 & PSN@1 & PSN@3 & PSN@5\\
\hline
\multicolumn{1}{|c||}{\multirow{1}{*}{LPSR}}&3.24&3.42 &3.64 &3.24 & 3.37& 3.52\\
\hline
% \multicolumn{1}{|c||}{\multirow{1}{*}{PLT}}&45.37& 38.94 & 35.88 & -& -& -\\
% \hline
% \multicolumn{1}{|c||}{\multirow{1}{*}{GBDT-S}}&42.11& 39.06 & - & -& -& -\\
% \hline
% \multicolumn{1}{|c||}{\multirow{1}{*}{CRAFTML}}&\textbf{47.87}& \textbf{41.28} & 38.01& -& -& -\\
% \hline
\multicolumn{1}{|c||}{\multirow{1}{*}{FastXML}}&6.48& 7.52& 8.31& 6.51& 7.26& 7.79\\
\hline
\multicolumn{1}{|c||}{\multirow{1}{*}{PFastreXML}}&3.15& 3.87& 4.43& 3.15& 3.68& 4.06 \\
\hline
\hline
\multicolumn{1}{|c||}{\multirow{1}{*}{LdSM}}&\textbf{7.16} &\textbf{8.26} & \textbf{9.11} &\textbf{7.16} &\textbf{7.92} & \textbf{8.45}\\
\hline
\end{tabular}
\centering
\vspace{0.1in}

\setlength{\tabcolsep}{0.5pt}
\begin{tabular}{c|| c| c| c||c| c| c|}
\hline
\multicolumn{7}{|c|}{Amazon-670k $D = 135k, K = 670k$}\\
% N = 197k,
\hline
\multicolumn{1}{|c||}{\multirow{1}{*}{Algorithm}} &PSP@1 & PSP@3 & PSP@5 & PSN@1 & PSN@3 & PSN@5 \\
\hline
\multicolumn{1}{|c||}{\multirow{1}{*}{LPSR}} &16.68&18.07  &19.43 &16.68 & 17.70& 18.63\\
\hline
% \multicolumn{1}{|c||}{\multirow{1}{*}{PLT}} &36.65& 32.12 & 28.85 & -& -& -\\
% \hline
% \multicolumn{1}{|c||}{\multirow{1}{*}{GBDT-S}} &-& - & - & -& -& -\\
% \hline
% \multicolumn{1}{|c||}{\multirow{1}{*}{CRAFTML}} &37.35& 33.31 & 30.62 & -& -& -\\
% \hline
\multicolumn{1}{|c||}{\multirow{1}{*}{FastXML}} &19.37& 23.26& 26.85& 19.37& 22.25& 24.69	\\
\hline
\multicolumn{1}{|c||}{\multirow{1}{*}{PFastreXML}} &\textbf{29.30}& 30.80& 32.43& \textbf{29.30}& \textbf{30.40}& \textbf{31.49} \\
\hline
\hline
\multicolumn{1}{|c||}{\multirow{1}{*}{LdSM}} &28.14& \textbf{30.82}& \textbf{33.16}& 28.14& 29.80& 30.71\\
\hline
\end{tabular}
% \centering
% \vspace{0.025in}
% \setlength{\tabcolsep}{0.5pt}
% \begin{tabular}{c|| c| c| c||c| c| c|}
% \hline
% \multicolumn{7}{|c|}{(8) WikiLSHTC-325k $D = 1.6M, K = 325k$}\\
% \hline
% \multicolumn{1}{|c||}{\multirow{1}{*}{Algorithm}}&PSP@1 & PSP@3 & PSP@5 & PSN@1 & PSN@3 & PSN@5 \\
% \hline
% \multicolumn{1}{|c||}{\multirow{1}{*}{LPSR}} & 6.93&7.21  &7.86&6.93 & 7.11& 7.46\\
% \hline
% % \multicolumn{1}{|c||}{\multirow{1}{*}{PLT}} & 45.67& 29.13 & 21.95 & -& -& -\\
% % \hline
% % \multicolumn{1}{|c||}{\multirow{1}{*}{GBDT-S}} & -& - & - & -& -& -\\
% % \hline
% % \multicolumn{1}{|c||}{\multirow{1}{*}{CRAFTML}} & \textbf{56.57}& 34.73 & 25.03& -& -& -\\
% % \hline
% \multicolumn{1}{|c||}{\multirow{1}{*}{FastXML}} & 16.35&  20.99& 23.56 &16.35 & 19.56& 21.02	\\
% \hline
% \multicolumn{1}{|c||}{\multirow{1}{*}{PFastreXML}} & 30.66&  31.55&33.12& 30.66 & 31.24& 32.09 \\
% \hline\hline
% \multicolumn{1}{|c||}{\multirow{1}{*}{LdSM}} & \textbf{30.88} &\textbf{33.65} & \textbf{34.59} &\textbf{30.88} &\textbf{33.31}&\textbf{33.92}\\
% \hline
% \end{tabular}
% % \hfill

\end{table}

\begin{table}[htp!]
\scriptsize
% \vspace{-0.2in}
\setlength{\tabcolsep}{5pt}
\caption{Precisions: $P@1$, $P@3$, and $P@5$ ($\%$) and nDCG scores: $N@1$, $N@3$, and $N@5$ ($\%$) obtained for tree-based approaches: GBDT-S, CRAFTML, FastXML, PFastreXML, LPSR, PLT, and LdSM and other (not purely tree-based) methods: Parabel, DisMEC, PD-Sparse, PPD-Sparse, OVA-Primal++, LEML, and SLEEC,  on various data sets. The best result among tree-based methods is in bold, and among all methods is underlined.}
% \vspace{0.05in}
% % \centering
% \vspace{0.05in}
\begin{tabular}{c |c|| c| c| c||c| c| c|}
\multirow{2}{*}{}&
\multicolumn{7}{c}{Mediamill }\\
\cline{2-8}
&\multicolumn{1}{|c||}{\multirow{1}{*}{Algorithm}} & P@1 & P@3 & P@5 & N@1 & N@3 & N@5\\
\cline{2-8}
\multirow{6}{*}{\rotatebox{90}{Other} $\begin{dcases*}\\ \\ \\ \\ \\ \end{dcases*}$}&\multicolumn{1}{|c||}{\multirow{1}{*}{Parabel}} & 83.91& 67.12 & 52.99 & 83.91& 75.22 & 72.21\\
\cline{2-8}
&\multicolumn{1}{|c||}{\multirow{1}{*}{DiSMEC}} & -& - & - & -& -& -\\
\cline{2-8}
&\multicolumn{1}{|c||}{\multirow{1}{*}{PD-Sparse}} & 81.86& 62.52 & 45.11 & 81.86& 70.21 & 63.71\\
\cline{2-8}
&\multicolumn{1}{|c||}{\multirow{1}{*}{PPD-Sparse}} & -& - & - & -& - & -\\
\cline{2-8}
&\multicolumn{1}{|c||}{\multirow{1}{*}{OVA-Primal}} & -& - & - & -& - & -\\
\cline{2-8}
&\multicolumn{1}{|c||}{\multirow{1}{*}{LEML}} & 84.01& 67.20 & 52.80 & 84.01& 75.23 & 71.96\\
\cline{2-8}
&\multicolumn{1}{|c||}{\multirow{1}{*}{SLEEC}} & 87.82& 73.45& 59.17 & 87.82& 81.50 & 79.22\\
\hline
\hline
\multirow{6}{*}{\rotatebox{90}{Tree} $\begin{dcases*}\\ \\ \\ \\ \\ \end{dcases*}$}&\multicolumn{1}{|c||}{\multirow{1}{*}{LPSR}} & 83.57& 65.78 & 49.97 & 83.57& 74.06 & 69.34\\
\cline{2-8}
&\multicolumn{1}{|c||}{\multirow{1}{*}{PLT}} & -& - & - & -& -& -\\
\cline{2-8}
&\multicolumn{1}{|c||}{\multirow{1}{*}{GBDT-S}} & 84.23& 67.85 & - & -& -& -\\
\cline{2-8}
&\multicolumn{1}{|c||}{\multirow{1}{*}{CRAFTML}} & 85.86& 69.01 & 54.65 & -& -& -\\
\cline{2-8}
&\multicolumn{1}{|c||}{\multirow{1}{*}{FastXML}} & 84.22& 67.33 & 53.04 &84.22 & 75.41& 72.37\\
\cline{2-8}
&\multicolumn{1}{|c||}{\multirow{1}{*}{PFastreXML}} &83.98 & 67.37 & 53.02 & 83.98& 75.31&72.21 \\
\hline
\hline
&\multicolumn{1}{|c||}{\multirow{1}{*}{LdSM}}& \underline{\textbf{90.64}} &\underline{ \textbf{73.60}} & \underline{\textbf{58.62}} &\underline{\textbf{90.64}} &\underline{\textbf{82.14}}&\underline{\textbf{79.23} }\\
\cline{2-8}
\end{tabular}
% \par
% \bigskip
% \vspace{-0.05in}
\vspace{0.05in}
\hfill
\begin{tabular}{c |c|| c| c| c||c| c| c|}
\multirow{2}{*}{}&
\multicolumn{7}{c}{Bibtex }\\
\cline{2-8}
&\multicolumn{1}{|c||}{\multirow{1}{*}{Algorithm}} & P@1 & P@3 & P@5 & N@1 & N@3 & N@5 \\
\cline{2-8}
\multirow{6}{*}{\rotatebox{90}{Other} $\begin{dcases*}\\ \\ \\ \\ \\ \end{dcases*}$}&\multicolumn{1}{|c||}{\multirow{1}{*}{Parabel}} & 64.53& 38.56 & 27.94 & 64.53& 59.35 & 61.06\\
\cline{2-8}
&\multicolumn{1}{|c||}{\multirow{1}{*}{DiSMEC}} & -& - & - & -& -& -\\
\cline{2-8}
&\multicolumn{1}{|c||}{\multirow{1}{*}{PD-Sparse}} & 61.29& 35.82 & 25.74 & 61.29& 55.83 & 57.35\\
\cline{2-8}
&\multicolumn{1}{|c||}{\multirow{1}{*}{PPD-Sparse}} & -& - & - & -& - & -\\
\cline{2-8}
&\multicolumn{1}{|c||}{\multirow{1}{*}{OVA-Primal}} & -& - & - & -& - & -\\
\cline{2-8}
&\multicolumn{1}{|c||}{\multirow{1}{*}{LEML}} & 62.54& 38.41 & 28.21& 62.54& 58.22 & 60.53\\
\cline{2-8}
&\multicolumn{1}{|c||}{\multirow{1}{*}{SLEEC}} & 65.08& 39.64 & 28.87 & \underline{65.08}& \underline{60.47} & 62.64\\
\hline
\hline
\multirow{6}{*}{\rotatebox{90}{Tree} $\begin{dcases*}\\ \\ \\ \\ \\ \end{dcases*}$}&\multicolumn{1}{|c||}{\multirow{1}{*}{LPSR}} & 62.11& 36.65 & 26.53 &62.11 & 56.50& 58.23\\
\cline{2-8}
&\multicolumn{1}{|c||}{\multirow{1}{*}{PLT}} & -& - & - & -& -& -\\
\cline{2-8}
&\multicolumn{1}{|c||}{\multirow{1}{*}{GBDT-S}} & -& - & - & -& -& -\\
\cline{2-8}
&\multicolumn{1}{|c||}{\multirow{1}{*}{CRAFTML}}& \underline{\textbf{65.15}}& \underline{\textbf{39.83}} & 28.99& -& - &-\\
\cline{2-8}
&\multicolumn{1}{|c||}{\multirow{1}{*}{FastXML}} & 63.42& 39.23 & 28.86 & 63.42& 59.51&61.70 \\
\cline{2-8}
&\multicolumn{1}{|c||}{\multirow{1}{*}{PFastreXML}} & 63.46 & 39.22 & 29.14 & 63.46& 59.61& 62.12\\
\hline
\hline
&\multicolumn{1}{|c||}{\multirow{1}{*}{LdSM}} & 64.69 & 39.70 & \underline{\textbf{29.25}} &\textbf{64.69} &\textbf{60.37} &\underline{\textbf{62.73}}  \\
\cline{2-8}
\end{tabular}
% \par
% \bigskip
% %\vspace{-0.05in}
% \vspace{0.05in}
\hfill
\begin{tabular}{c |c|| c| c| c||c| c| c|}
\multirow{2}{*}{}&
\multicolumn{7}{c}{Delicious}\\
\cline{2-8}
&\multicolumn{1}{|c||}{\multirow{1}{*}{Algorithm}} & P@1 & P@3 & P@5 & N@1& N@3 & N@5\\
\cline{2-8}
\multirow{6}{*}{\rotatebox{90}{Other} $\begin{dcases*}\\ \\ \\ \\ \\ \end{dcases*}$}&\multicolumn{1}{|c||}{\multirow{1}{*}{Parabel}} & 67.44& 61.83 & 56.75 & 67.44& 63.15 & 59.41\\
\cline{2-8}
&\multicolumn{1}{|c||}{\multirow{1}{*}{DiSMEC}} & -& - & - & -& -& -\\
\cline{2-8}
&\multicolumn{1}{|c||}{\multirow{1}{*}{PD-Sparse}} & 51.82& 44.18 & 38.95 & 51.82& 46.00
 & 42.02\\
\cline{2-8}
&\multicolumn{1}{|c||}{\multirow{1}{*}{PPD-Sparse}} & -& - & - & -& - & -\\
\cline{2-8}
&\multicolumn{1}{|c||}{\multirow{1}{*}{OVA-Primal}} & -& - & - & -& - & -\\
\cline{2-8}
&\multicolumn{1}{|c||}{\multirow{1}{*}{LEML}} & 65.67& 60.55& 56.08 & 65.67& 61.77 & 58.47\\
\cline{2-8}
&\multicolumn{1}{|c||}{\multirow{1}{*}{SLEEC}} & 67.59& 61.38 & 56.56 & 67.59& 62.87 & 59.28\\
\hline
\hline
\multirow{6}{*}{\rotatebox{90}{Tree} $\begin{dcases*}\\ \\ \\ \\ \\ \end{dcases*}$}&\multicolumn{1}{|c||}{\multirow{1}{*}{LPSR}} & 65.01&  58.96& 53.49 &65.01 & 60.45& 56.38\\
\cline{2-8}
&\multicolumn{1}{|c||}{\multirow{1}{*}{PLT}} & -& - & - & -& -& -\\
\cline{2-8}
&\multicolumn{1}{|c||}{\multirow{1}{*}{GBDT-S}} & 69.29& 63.62 & - & -& -& -\\
\cline{2-8}
&\multicolumn{1}{|c||}{\multirow{1}{*}{CRAFTML}} & 70.26& 63.98 & 59.00 & -& -& -\\
\cline{2-8}
&\multicolumn{1}{|c||}{\multirow{1}{*}{FastXML}} &69.61 & 64.12 & 59.27 &69.61 & 65.47&61.90 \\
\cline{2-8}
&\multicolumn{1}{|c||}{\multirow{1}{*}{PFastreXML}} & 67.13&  62.33& 58.62 &67.13 & 63.48& 60.74\\
\hline\hline
&\multicolumn{1}{|c||}{\multirow{1}{*}{LdSM}} & \underline{\textbf{71.91}} & \underline{\textbf{65.34}} & \underline{\textbf{60.24}} &\underline{\textbf{71.91} }& \underline{\textbf{66.90}}&\underline{\textbf{63.09}} \\
\cline{2-8}
\end{tabular}
% \par
% \bigskip
% %\vspace{-0.05in}
\vspace{0.05in}
\hfill
\hfill
\begin{tabular}{c |c|| c| c| c||c| c| c|}
\multirow{2}{*}{}&\multicolumn{7}{c}{AmazonCat-13k }\\
\cline{2-8}
&\multicolumn{1}{|c||}{\multirow{1}{*}{Algorithm}} & P@1 & P@3 & P@5 & N@1 & N@3 & N@5 \\
\cline{2-8}
\multirow{6}{*}{\rotatebox{90}{Other} $\begin{dcases*}\\ \\ \\ \\ \\ \end{dcases*}$}&\multicolumn{1}{|c||}{\multirow{1}{*}{Parabel}} & 93.03& \underline{79.16}& \underline{64.52} &93.03& \underline{87.72} & \underline{86.00}\\
\cline{2-8}
&\multicolumn{1}{|c||}{\multirow{1}{*}{DiSMEC}} & 93.40& 79.10& 64.10& 93.40& 87.70& 85.80\\
\cline{2-8}
&\multicolumn{1}{|c||}{\multirow{1}{*}{PD-Sparse}} & 90.60& 75.14 & 60.69 & 90.60& 84.00 & 82.05\\
\cline{2-8}
&\multicolumn{1}{|c||}{\multirow{1}{*}{PPD-Sparse}} & -& - & - & -& - & -\\
\cline{2-8}
&\multicolumn{1}{|c||}{\multirow{1}{*}{OVA-Primal}} & 93.75& 78.89 & 63.66 & -& - & -\\
\cline{2-8}
&\multicolumn{1}{|c||}{\multirow{1}{*}{LEML}} & -& - & - & -& - & -\\
\cline{2-8}
&\multicolumn{1}{|c||}{\multirow{1}{*}{SLEEC}} & 90.53& 76.33 & 61.52& 90.53& 84.96 & 82.77\\
\hline
\hline
\multirow{6}{*}{\rotatebox{90}{Tree} $\begin{dcases*}\\ \\ \\ \\ \\ \end{dcases*}$}
&\multicolumn{1}{|c||}{\multirow{1}{*}{LPSR}} & -& - & - & -& -& -\\
\cline{2-8}
&\multicolumn{1}{|c||}{\multirow{1}{*}{PLT}} & 91.47& 75.84 & 61.02 & -& -& -\\
\cline{2-8}
&\multicolumn{1}{|c||}{\multirow{1}{*}{GBDT-S}} & -& - & - & -& -& -\\
\cline{2-8}
&\multicolumn{1}{|c||}{\multirow{1}{*}{CRAFTML}} & 92.78& \textbf{78.48} & 63.58 & -& -& -\\
\cline{2-8}
&\multicolumn{1}{|c||}{\multirow{1}{*}{FastXML}} & 93.11&  78.2& 63.41 &93.11 & \textbf{87.07}&\textbf{85.16} \\
\cline{2-8}
&\multicolumn{1}{|c||}{\multirow{1}{*}{PFastreXML}} &91.75 &  77.97& \textbf{63.68} & 91.75& 86.48&84.96 \\
\hline\hline
&\multicolumn{1}{|c||}{\multirow{1}{*}{LdSM}} & \underline{\textbf{93.87}}&75.41 &57.86 & \underline{\textbf{93.87}}& 85.06&80.63 \\
\cline{2-8}
\end{tabular}
% \par
% \bigskip
% %\vspace{-0.05in}
% \vspace{0.1in}
% \hfill
\begin{tabular}{c |c|| c| c| c||c| c| c|}
\multirow{2}{*}{}&\multicolumn{7}{c}{Wiki10-31k}\\
\cline{2-8}
&\multicolumn{1}{|c||}{\multirow{1}{*}{Algorithm}} & P@1 & P@3 & P@5 & N@1 & N@3 & N@5 \\
\cline{2-8}
\multirow{6}{*}{\rotatebox{90}{Other} $\begin{dcases*}\\ \\ \\ \\ \\ \end{dcases*}$}&\multicolumn{1}{|c||}{\multirow{1}{*}{Parabel}} & 84.31& 72.57 & 63.39 & 83.03& 71.01 & 68.30\\
\cline{2-8}
&\multicolumn{1}{|c||}{\multirow{1}{*}{DiSMEC}} & 85.20& 74.60 & 65.90 & 84.10& \underline{77.10}& \underline{70.40}\\
\cline{2-8}
&\multicolumn{1}{|c||}{\multirow{1}{*}{PD-Sparse}} & -& - & - & -& - & -\\
\cline{2-8}
&\multicolumn{1}{|c||}{\multirow{1}{*}{PPD-Sparse}} & -& - & - & -& - & -\\
\cline{2-8}
&\multicolumn{1}{|c||}{\multirow{1}{*}{OVA-Primal}} & 84.17& \underline{74.73} & \underline{65.92} & -& - & -\\
\cline{2-8}
&\multicolumn{1}{|c||}{\multirow{1}{*}{LEML}} & 73.47& 62.43 & 54.35 & 73.47& 64.92 & 58.69\\
\cline{2-8}
&\multicolumn{1}{|c||}{\multirow{1}{*}{SLEEC}} & \underline{85.88}& 72.98 & 62.70 &\underline{ 85.88}& 76.02& 68.13\\
\hline
\hline
\multirow{6}{*}{\rotatebox{90}{Tree} $\begin{dcases*}\\ \\ \\ \\ \\ \end{dcases*}$}&
\multicolumn{1}{|c||}{\multirow{1}{*}{LPSR}} & 72.72&58.51  & 49.50 &72.72 & 61.71& 54.63\\
\cline{2-8}
&\multicolumn{1}{|c||}{\multirow{1}{*}{PLT}} & 84.34& 72.34& 62.72& -& -& -\\
\cline{2-8}
&\multicolumn{1}{|c||}{\multirow{1}{*}{GBDT-S}} & 84.34& 70.82 & - & -& -& -\\
\cline{2-8}
&\multicolumn{1}{|c||}{\multirow{1}{*}{CRAFTML}} & \textbf{85.19}& \textbf{73.17} & \textbf{63.27} & -& -& -\\
\cline{2-8}
&\multicolumn{1}{|c||}{\multirow{1}{*}{FastXML}} & 83.03&  67.47& 57.76 &83.03 & \textbf{75.35}& 63.36\\
\cline{2-8}
&\multicolumn{1}{|c||}{\multirow{1}{*}{PFastreXML}} & 83.57&  68.61& 59.10 & 83.57& 72.00&64.54 \\
\hline\hline
&\multicolumn{1}{|c||}{\multirow{1}{*}{LdSM}} & 83.74 & 71.74 & 61.51 &\textbf{83.74} &74.60 & \textbf{66.77}\\
\cline{2-8}
\end{tabular}
% \par
% \bigskip
% %\vspace{-0.05in}
% \vspace{0.05in}
\vspace{0.1in}
\hfill
\begin{tabular}{c |c|| c| c| c||c| c| c|}
\multirow{2}{*}{}&\multicolumn{7}{c}{Delicious-200k}\\
\cline{2-8}
&\multicolumn{1}{|c||}{\multirow{1}{*}{Algorithm}} & P@1 & P@3 & P@5 & N@1 & N@3 & N@5 \\
\cline{2-8}
\multirow{6}{*}{\rotatebox{90}{Other} $\begin{dcases*}\\ \\ \\ \\ \\ \end{dcases*}$}&\multicolumn{1}{|c||}{\multirow{1}{*}{Parabel}} & 46.97& 40.08 & 36.63 & 46.97& 41.72 & 39.07\\
\cline{2-8}
&\multicolumn{1}{|c||}{\multirow{1}{*}{DiSMEC}} & 45.50& 38.70 & 35.50 & 45.50& 40.90& 37.80\\
\cline{2-8}
&\multicolumn{1}{|c||}{\multirow{1}{*}{PD-Sparse}} & 34.37 & 29.48& 27.04& 34.37& 30.60 & 28.65\\
\cline{2-8}
&\multicolumn{1}{|c||}{\multirow{1}{*}{PPD-Sparse}} & -& - & - & -& - & -\\
\cline{2-8}
&\multicolumn{1}{|c||}{\multirow{1}{*}{OVA-Primal}} & -& - & - & -& - & -\\
\cline{2-8}
&\multicolumn{1}{|c||}{\multirow{1}{*}{LEML}} & 40.73& 37.71	 & 35.84& 40.73& 38.44 & 37.01\\
\cline{2-8}
&\multicolumn{1}{|c||}{\multirow{1}{*}{SLEEC}} & 47.85& \underline{42.21} & \underline{39.43} &\underline{47.85}&\underline{43.52}&\underline{41.37}\\
\hline
\hline
\multirow{6}{*}{\rotatebox{90}{Tree} $\begin{dcases*}\\ \\ \\ \\ \\ \end{dcases*}$}&
\multicolumn{1}{|c||}{\multirow{1}{*}{LPSR}} & 18.59&15.43  &14.07 &18.59 & 16.17& 15.13\\
\cline{2-8}
&\multicolumn{1}{|c||}{\multirow{1}{*}{PLT}} & 45.37& 38.94 & 35.88 & -& -& -\\
\cline{2-8}
&\multicolumn{1}{|c||}{\multirow{1}{*}{GBDT-S}} & 42.11& 39.06 & - & -& -& -\\
\cline{2-8}
&\multicolumn{1}{|c||}{\multirow{1}{*}{CRAFTML}} &\underline{ \textbf{47.87}}& \textbf{41.28} & 38.01 & -& -& -\\
\cline{2-8}
&\multicolumn{1}{|c||}{\multirow{1}{*}{FastXML}} & 43.07&  38.66& 36.19 &43.07 & 39.70& 37.83\\
\cline{2-8}
&\multicolumn{1}{|c||}{\multirow{1}{*}{PFastreXML}} & 41.72&  37.83&35.58& 41.72 & 38.76& 37.08 \\
\hline\hline
&\multicolumn{1}{|c||}{\multirow{1}{*}{LdSM}} & 45.26 &40.53 & \textbf{38.23} &\textbf{45.26} &\textbf{41.66} & \textbf{39.79}\\
\cline{2-8}
\end{tabular}
\begin{tabular}{c |c|| c| c| c||c| c| c|}
\multirow{2}{*}{}&\multicolumn{7}{c}{Amazon-670k}\\
\cline{2-8}
&\multicolumn{1}{|c||}{\multirow{1}{*}{Algorithm}} & P@1 & P@3 & P@5 & N@1 & N@3 & N@5 \\
\cline{2-8}
\multirow{6}{*}{\rotatebox{90}{Other} $\begin{dcases*}\\ \\ \\ \\ \\ \end{dcases*}$}&\multicolumn{1}{|c||}{\multirow{1}{*}{Parabel}} & 44.89& 39.80 & 36.00& \underline{44.89}& \underline{42.14} & 40.36\\
\cline{2-8}
&\multicolumn{1}{|c||}{\multirow{1}{*}{DiSMEC}} & 44.70& 39.70 & 36.10& 44.70& 42.10& \underline{40.50}\\
\cline{2-8}
&\multicolumn{1}{|c||}{\multirow{1}{*}{PD-Sparse}} & -& - & - & -& - & -\\
\cline{2-8}
&\multicolumn{1}{|c||}{\multirow{1}{*}{PPD-Sparse}} & \underline{45.32}& \underline{40.37} & \underline{36.92} & -& - & -\\
\cline{2-8}
&\multicolumn{1}{|c||}{\multirow{1}{*}{OVA-Primal}} & -& - & - & -& - & -\\
\cline{2-8}
&\multicolumn{1}{|c||}{\multirow{1}{*}{LEML}} & 8.13& 6.83 & 6.03 & 8.13& 7.30& 6.85\\
\cline{2-8}
&\multicolumn{1}{|c||}{\multirow{1}{*}{SLEEC}} & 35.05& 31.25& 28.56 & 34.77& 32.74 & 31.53\\
\hline
\hline
\multirow{6}{*}{\rotatebox{90}{Tree} $\begin{dcases*}\\ \\ \\ \\ \\ \end{dcases*}$}&\multicolumn{1}{|c||}{\multirow{1}{*}{LPSR}} & 28.65&24.88  &22.37 &28.65 & 26.40& 25.03\\
\cline{2-8}
&\multicolumn{1}{|c||}{\multirow{1}{*}{PLT}} & 36.65& 32.12 & 28.85 & -& -& -\\
\cline{2-8}
&\multicolumn{1}{|c||}{\multirow{1}{*}{GBDT-S}} & -& - & - & -& -& -\\
\cline{2-8}
&\multicolumn{1}{|c||}{\multirow{1}{*}{CRAFTML}} & 37.35& 33.31 & 30.62 & -& -& -\\
\cline{2-8}
&\multicolumn{1}{|c||}{\multirow{1}{*}{FastXML}} & 36.99&  33.28& 30.53 &36.99 & 35.11& 33.86	\\
\cline{2-8}
&\multicolumn{1}{|c||}{\multirow{1}{*}{PFastreXML}} & 39.46&  35.81&33.05& 39.46 & 37.78& 36.69 \\
\hline\hline
&\multicolumn{1}{|c||}{\multirow{1}{*}{LdSM}} & \textbf{42.63} & \textbf{38.09} & \textbf{34.70} &\textbf{42.63} &\textbf{40.37} & \textbf{38.89}\\
\cline{2-8}
\end{tabular}
\label{tab:accuall}
\end{table}

\begin{figure}[htp!]
  \begin{center}
Delicious\\
\hspace{-0.4in}$M = 2$ \hspace{2.2in}$M = 4$\\
\includegraphics[width=0.47\textwidth]{Figures/deli-ens-m2.png}
\includegraphics[width=0.47\textwidth]{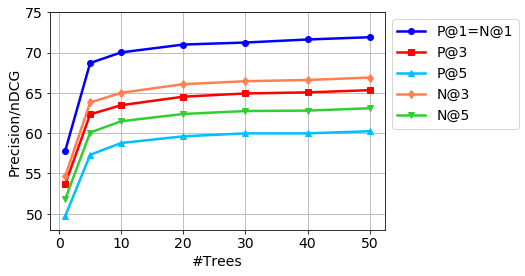}
Bibtex\\
\hspace{-0.4in}$M = 2$ \hspace{2.2in}$M = 4$\\
\includegraphics[width=0.47\textwidth]{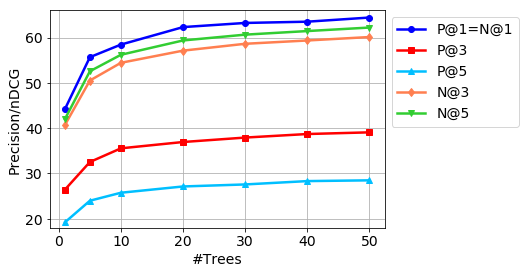}
\includegraphics[width=0.47\textwidth]{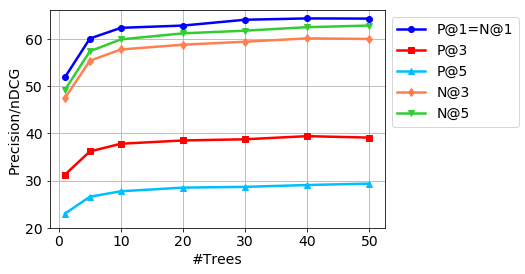}
Mediamill\\
\hspace{-0.4in}$M = 2$ \hspace{2.2in}$M = 4$\\
\includegraphics[width=0.47\textwidth]{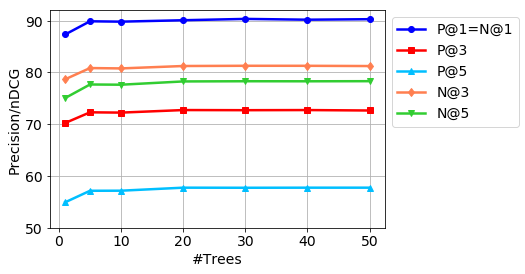}
\includegraphics[width=0.47\textwidth]{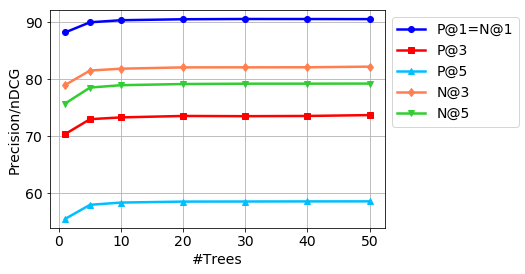}\\
Wiki10\\
$M = 4$\\
\includegraphics[width=0.47\textwidth]{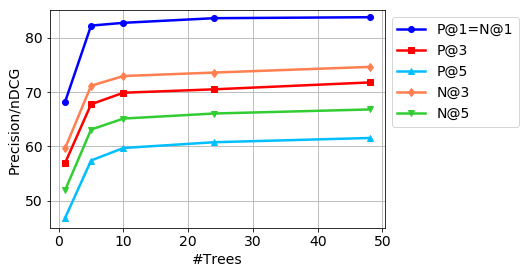}
\end{center}
% \vspace{-0.15in}
\caption{The behavior of Precision/nDCG score as a function of the number of trees in the ensemble. Plots were obtained for Delicious, Bibtex, Mediamill, and Wiki10 data sets.}
\label{fig:impex2}
\end{figure}

\begin{figure}[htp!]
  \begin{center}
Delicious\\
\hspace{-0.4in}$M = 2$ \hspace{2.2in}$M = 4$\\
\includegraphics[width=0.47\textwidth]{Figures/delim2.png}
\includegraphics[width=0.47\textwidth]{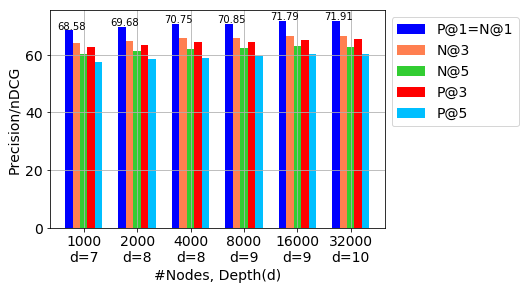}
Bibtex\\
\hspace{-0.4in}$M = 2$ \hspace{2.2in}$M = 4$\\
\includegraphics[width=0.47\textwidth]{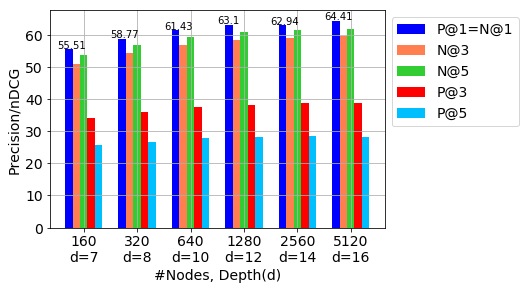}
\includegraphics[width=0.47\textwidth]{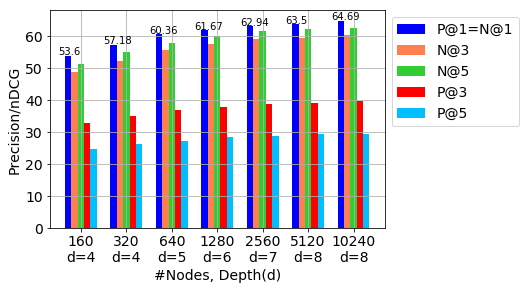}
Mediamill\\
\hspace{-0.4in}$M = 2$ \hspace{2.2in}$M = 4$\\
\includegraphics[width=0.47\textwidth]{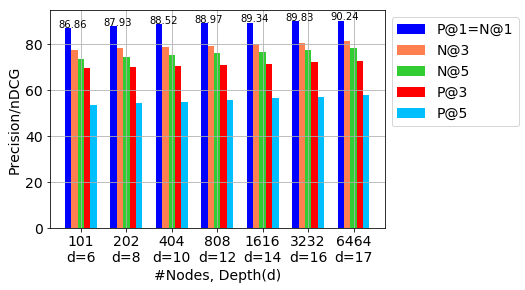}
\includegraphics[width=0.47\textwidth]{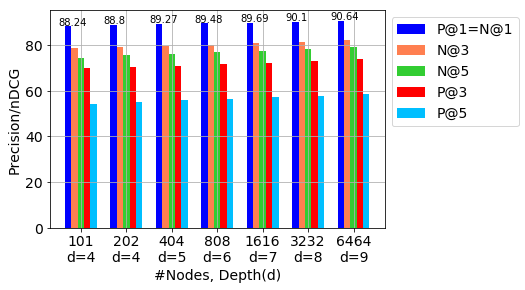}\\
\hspace{-0.5in}AmazonCat \hspace{2in}Wiki10\\
\hspace{-0.4in}$M = 2$ \hspace{2.2in}$M = 4$\\
\includegraphics[width=0.47\textwidth]{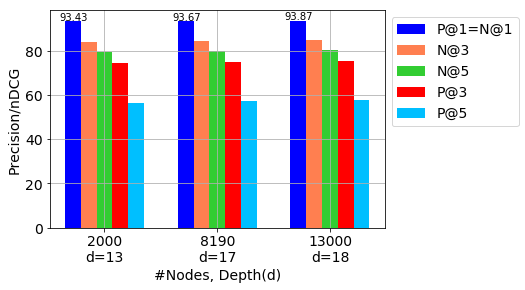}
\includegraphics[width=0.47\textwidth]{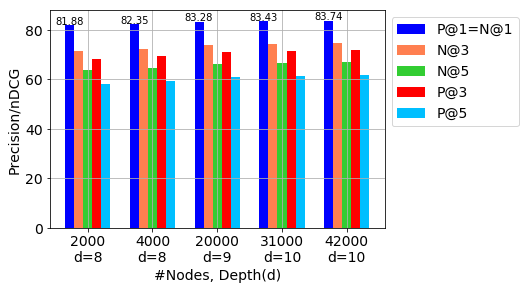}\\
\end{center}
% \vspace{-0.15in}
\caption{The behavior of Precision/nDCG score as a function of the number of nodes $T_{max}$ (including leaves) and tree depth of the deepest tree in the ensemble. Plots were obtained for Delicious, Bibtex, Mediamill, AmazonCat, and Wiki10 data sets.}
\label{fig:impex3}
\end{figure}

\newpage
\begin{figure}[htp!]
  \begin{center}
Bibtex\\
\includegraphics[width=0.49\textwidth]{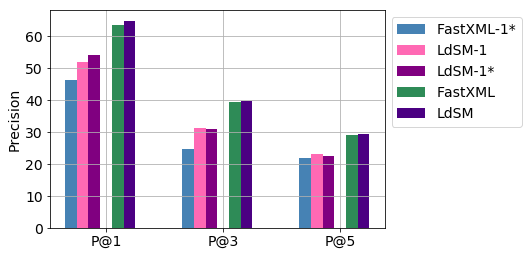}
\includegraphics[width=0.49\textwidth]{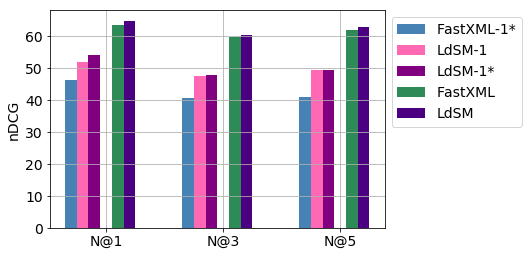}\\
Mediamill\\
\includegraphics[width=0.49\textwidth]{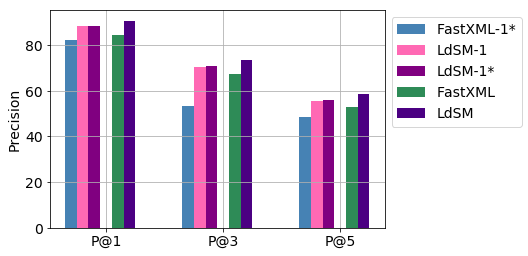}
\includegraphics[width=0.49\textwidth]{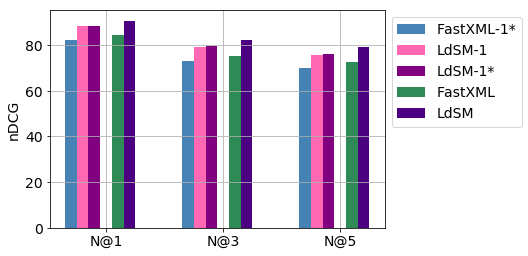}\\
Delicious\\
\includegraphics[width=0.49\textwidth]{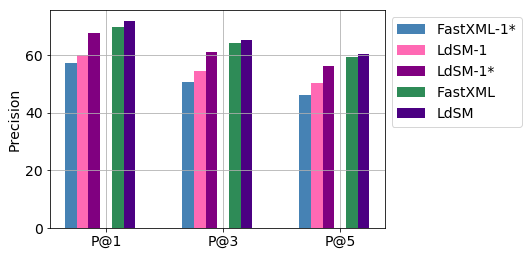}
\includegraphics[width=0.49\textwidth]{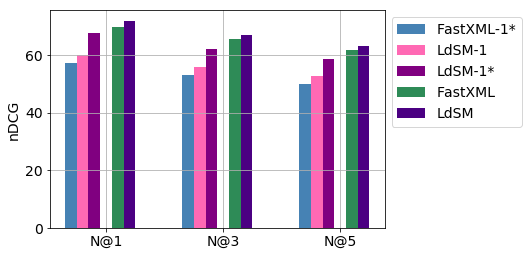}
\end{center}
\vspace{-0.15in}
\caption{The comparison of Precision (\textbf{left column}) and nDCG (\textbf{right column}) score for LdSM and FastXML working in the ensemble (\textbf{right bars}) as well as for single-tree (\textbf{left bars}) (LdSM-1: exemplary tree chosen from LdSM ensemble, LdSM-1$^{*}$, FastXML-1$^*$: optimal single trees). Plots were obtained for Bibtex, Mediamill and Delicious data sets.}
\label{fig:impex5}
\end{figure}

\end{document}